\documentclass[11pt]{article}

\usepackage[utf8]{inputenc} % allow utf-8 input
\usepackage[T1]{fontenc}    % use 8-bit T1 fonts
\usepackage{hyperref}       % hyperlinks
\usepackage{url}            % simple URL typesetting
\usepackage{booktabs}       % professional-quality tables
\usepackage{amsfonts}       % blackboard math symbols
\usepackage{nicefrac}       % compact symbols for 1/2, etc.
\usepackage{microtype}      % microtypography

\usepackage{microtype}
\usepackage{graphicx}
\usepackage{subfigure}
\usepackage{booktabs}
\usepackage{algorithm}
\usepackage{xcolor}
\usepackage{threeparttable}

\usepackage{algorithmic,eqparbox,array}

\usepackage{lipsum,etoolbox}

\usepackage{enumitem}
\setlist[itemize]{leftmargin=1cm}
\setlist[enumerate]{leftmargin=1cm}

\usepackage%[hidelinks,colorlinks,citecolor={black},linkcolor=black]
{hyperref}
\usepackage{amsmath}
\usepackage{amssymb}
\usepackage{mathtools}
\usepackage{amsthm}

\usepackage[margin = 1in]{geometry}

\newtheorem{theorem}{Theorem}[section]

\newtheorem{lemma}[theorem]{Lemma}

\newtheorem{definition}[theorem]{Definition}
\newtheorem{assumption}[theorem]{Assumption}
\theoremstyle{remark}
\newtheorem{remark}[theorem]{Remark}

\newcommand{\R}{\mathbb{R}}

\newcommand{\modelname}{RNN-ODE-Adap}
\newcommand{\Tr}{\text{Tr}}
\newcommand{\Te}{\text{Te}}

\usepackage[textsize=tiny]{todonotes}
\usepackage{authblk}

\graphicspath{{figs_arxiv/}}

\begin{document}

\title{Neural Differential Recurrent Neural Network with \\
Adaptive Time Steps}

\author[1]{Yixuan~Tan}
 \author[2]{Liyan~Xie}
\author[1]{Xiuyuan~Cheng\thanks{Email: xiuyuan.cheng@duke.edu.}}

\affil[1]{\small
Department of Mathematics, Duke University}

\affil[2]{\small
School of Data Science, The Chinese University of Hong Kong, Shenzhen}

\date{\vspace{-20pt}}
\maketitle

\begin{abstract}
The neural Ordinary Differential Equation (ODE) model has shown success in learning complex continuous-time processes from observations on discrete time stamps. In this work, we consider the modeling and forecasting of time series data that are non-stationary and may have sharp changes like spikes. We propose an RNN-based model, called \textit{\modelname}, that uses a neural ODE to represent the time development of the hidden states, and we adaptively select time steps based on the steepness of changes of the data over time so as to train the model more efficiently for the ``spike-like'' time series. Theoretically, \textit{\modelname} yields provably a consistent estimation of the intensity function for the Hawkes-type time series data. We also provide an approximation analysis of the RNN-ODE model showing the benefit of adaptive steps. The proposed model is demonstrated to achieve higher prediction accuracy with reduced computational cost on simulated dynamic system data and point process data and on a real electrocardiography dataset. 
\end{abstract}

\section{Introduction}

We consider the modeling and forecasting of time series characterized by {\it irregular} time steps and {\it non-stationary} patterns, which are commonly observed in various applications, such as finance \cite{gonzalez2012time} and healthcare \cite{cheng2015time}. 
We treat the data as a sequence of ordered observations from an unknown underlying continuous-time process, sampled at discrete time points. Recurrent Neural Network (RNN) \cite{rumelhart1986learning} is frequently employed to model such sequential data. This work proposes to use an RNN-based model with a neural Ordinary Differential Equation (ODE) to fit time series data. 

The classical neural ODE type approaches typically assume {\it regular} time grids in the data sequences, or the {\it same} irregular time grids across different sequences \cite{rubanova2019latent}.  
When dealing with highly non-stationary data, such as a time series with sudden spikes, it becomes imperative to select sufficiently small time steps to ensure accurate modeling of regions with steep changes. These regions require more refined time steps, especially the ones with abrupt spikes. However, it is common for most of the time horizon to observe a slow-varying and less steep time series, which is more ``flat'' over time; thus, the refined time steps would result in unnecessarily high computational costs. Some examples of time series with abrupt spikes are shown in Figure\,\ref{fig:spikes}, illustrating a spectrum ranging from continuous time series to discontinuous time series such as counting processes.

To train the neural ODE for data with non-stationary patterns more effectively,
we propose an approach that employs {\it adaptive time steps} in the neural ODE model, which we refer to as \textit{\modelname}. The model adaptively selects the time steps based on the local variation of the time series, enabling it to capture underlying trends with potentially fewer steps. Our numerical experiments showed that, compared to other baseline models that use regular time steps, \textit{\modelname} could achieve higher prediction accuracy with similar or lower time complexity.
The contribution of the work is as follows.

\begin{itemize}
    \item
    Based on a neural ODE model characterizing the dynamics of hidden states,
    we propose an algorithm to construct adaptive time steps, which assigns refined time steps to data around ``spikes'' while using rough time steps for data in ``flat'' segments. This can significantly reduce the computational cost in the training process with little impact on the modeling performance.
    
    \item We provide theoretical insights into the consistency of the model using the example of Hawkes process type data, and the approximation guarantee of the RNN-ODE model that illuminates the benefits of adaptive time steps.
    
    \item We conduct numerical experiments on both synthetic data and a real-world time-series data set to demonstrate the advantage of the proposed algorithm in terms of both modeling accuracy and computational efficiency. 
\end{itemize}

\begin{figure}[t]
    \centering
    \includegraphics[width=0.875\textwidth]{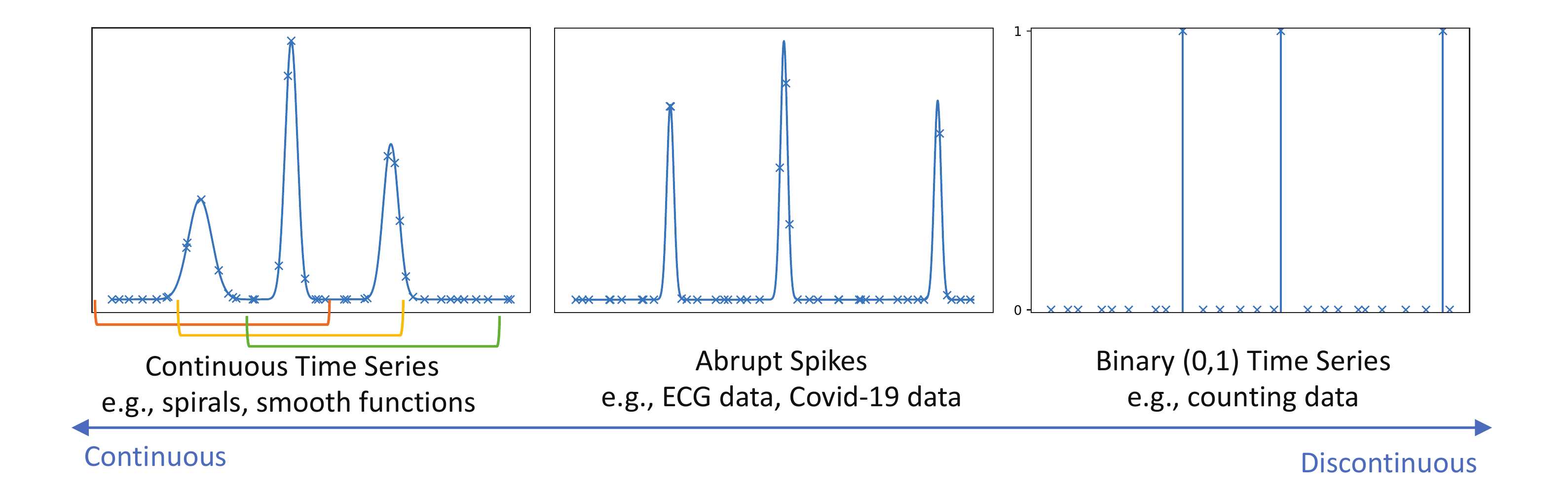}
    \caption{Illustration of spike-like time series. The crosses denote the discretely sampled time steps, which can be irregular.  In the left panel, the subsequences enclosed with the orange, yellow, and green brackets represent the (training or testing) windows generated from this sequence.}
    \label{fig:spikes}
\end{figure}

\subsection{Related Works} 
\paragraph{Neural ODE.}
Our work is closely related to the neural ODE \cite{chen2018neural} model, which parameterizes the derivative of the hidden state using a neural network. 
In \cite{chen2018neural}, a generative time-series model was proposed, which takes the neural ODE as the decoder.
Furthermore, \cite{rubanova2019latent} proposed a non-generative model with continuous-time hidden dynamics to handle irregularly sampled data based on \cite{chen2018neural}. 
Compared with existing works related to neural ODE \cite{rubanova2019latent, weinan2017proposal,zhang2019anodev2, dong2020towards, lu2018beyond, kidger2020neural, morrill2021neural, habiba2020neural, greydanus2021piecewise}, we model the ODE that determines the progression of hidden states by including the data itself in the derivative of the hidden state. 
In contrast to existing works on non-stationary environments such as the piecewise-constant ODE \cite{greydanus2021piecewise}, our work proposes to use adaptive time steps to automatically adapt to sparse spikes in the time series, without pre-defining the time period for each piece of ODE.

\paragraph{Neural CDE.}
 We note that the Neural Controlled Differential Equation (CDE) \cite{kidger2020neural_app} also incorporates the observations into the model continuously. Specifically, the hidden states in \cite{kidger2020neural_app} follow the CDE $h(t) = h(t_0) + \int_{t_0}^t f_\theta(h(s))\mathrm{d}X_s$, where the integral is a Riemann-Stieltjes integral. 
We would like to emphasize some key differences between model \eqref{eq:ode} and Neural CDE. The $X_s$ in Neural CDE is the natural cubic spline of $\{(x(t_i), t_i)\}_{i}$, and $f_\theta:\R^{d_{h}} \to \R^{d_{h}\times (D+1)}$, where $d_h$ is the number of hidden units and $D$ is the data dimension. Thus, for the same number of hidden units, Neural CDE requires a more complex parameterized $f_\theta$ to model $h(t)$. Moreover, since $X_s$ is obtained by cubic spline, it is less naturally adapted to the prediction task that requires extrapolation to the time stamps not seen when computing the spline. 
 Therefore, it is hard to evaluate the prediction performance of Neural CDE and thus we defer the evaluation under the Neural CDE setting for future work.

\paragraph{Continuous-Time RNNs.}  Our model belongs to the extensive family of continuous-time RNNs, originating from \cite{rosenblatt1961principles}. Several existing studies explore various RNN architectures, 
such as \cite{chang2019antisymmetricrnn, kag2020rnns, erichson2020lipschitz, rusch2020coupled, kag2021time}
These RNN models leverage their structures to address the exploding and vanishing gradient problem. 
Our model also adopts a continuous-time ODE framework for time series data, 
and the proposed adaptive time stamp selection method can be viewed as effectively reducing the length of the discrete sequence when a significant part of the process is changing slowly. 
Meanwhile, our approach can also be used concurrently with the methodologies such as in \cite{erichson2020lipschitz}. 
As the focus of our work is to model the ``spike-like'' time series data, the combination of our model and the existing continuous-time RNN models can further improve the efficiency when applied to such data.

\paragraph{Time Adaptivity.} Previous studies have investigated the incorporation of time adaptivity in continuous-time RNNs, such as GACTRNN \cite{heinrich2020learning}, TARNN \cite{kag2021time}, and LEM \cite{rusch2021long}. In these works, time adaptivity was incorporated 
by multiplying the ODE with an adaptively learned time modulator, usually parametrized by another sub-network.
In contrast, our method adaptively selects time steps during the preprocessing phase, where the selection process only utilizes the steepness of change of the time series data. 
Therefore, the proposed model does not involve the training of a sub-network for the time modulator as in the previous models, which may incur an increase in model size and additional computational costs.

\section{Problem Setup}\label{sec:setup}

\subsection{Training Data and Prediction Task}

Consider a random continuous time series $x(t)\in\R^D$ over the time horizon $[0,T]$ for some $T\in\R^+$. 
We observe multiple independent and identically distributed samples of the continuous process $x(t)$, where each sample is sampled at discrete time stamps, which can vary across different samples.
We split the observed sequences into training and testing sequences. From the training sequences, we generate a total of $K^{(\Tr)}$ training windows, denoted as $\{\boldsymbol{x}^{(\Tr,k)}\}_{k=1}^{K^{(\Tr)}}$, each of window length $N$ as our training data (see the left panel of Figure \ref{fig:adaptive_demo} for an illustrative example). Here $\boldsymbol{x}^{(\Tr,k)}\coloneqq \{ x^{(\Tr,k)}(t_{1}^{(\Tr, k)}),\ldots,x^{(\Tr,k)}(t_{N}^{(\Tr, k)})\}$ and $0<t_{1}^{(\Tr, k)}<\cdots<t_{N}^{(\Tr, k)}\leq T$ are the corresponding time stamps for the $k$-th training window. Similarly, we create $K^{(\Te)}$ testing windows of length $N$ from the testing sequences, denoted as $\{\boldsymbol{x}^{(\Te,k)}\}_{k=1}^{K^{(\Te)}}$. To simplify the notation, we may drop the superscripts and write $\{x(t_1),x(t_2),\ldots,x(t_N)\}$ for a given training window if it does not cause confusion.

Our goal is to make predictions based on historical data. 
Given a historical series $x(t_{1}), \dots, x(t_{n})$, we aim to perform either one-step or multi-step predictions. 
The one-step prediction involves predicting $x(t_{n+1})$ at a single future time $t_{n+1}$ based on $\{x(t_1) \dots, x(t_{n})\}$, while the $m$-step prediction includes forecasting $\{x(t_{n+1}), \dots, x(t_{n+m})\}$ at future times $t_{n+1}<\cdots<t_{n+m}$.
The detailed formulas for measuring the prediction accuracy are provided in Appendix \ref{sec:detail}.

\subsection{Training Objective}\label{sec:train-obj}

Given the time horizon $[0,T]$, recall that we have a collection of training windows with the $k$-th one denoted as   $\boldsymbol{x}^{(\Tr,k)}= \{ x^{(\Tr,k)}(t_{1}^{(\Tr,k)}),\ldots,x^{(\Tr,k)}(t_{N}^{(\Tr,k)})\}$. Here, the time steps are allowed to be {\it heterogeneous} for different training windows.
We train the model parametrized by neural networks (see Section \ref{sec:model} for the neural ODE model adopted in this paper) with trainable parameters $\Theta$ using the {\it mean-squared regression loss} function
\begin{equation}\label{eq:training_loss1}
\mathcal{L}(\Theta; \{\boldsymbol{x}^{(\Tr,k)}\}_{k=1}^{K^{(\Tr)}} ) = \sum_{k=1}^{K^{(\Tr)}}\sum_{i=1}^{N}  \| \hat{x}^{(\Tr,k)}(t_{i}^{(\Tr,k)}) - x^{(\Tr,k)}(t_{i}^{(\Tr,k)}) \|^2 |t_{i}^{(\Tr,k)} - t_{i-1}^{(\Tr,k)}|,   
\end{equation}
where $\Theta$ are the network parameters, $\hat x^{(k)}(t)$ is the output of the neural ODE model under parameter $\Theta$ conditioned on all past observation, and $t_{0}^{(\Tr,k)}$ is the added initial time stamp for each window.

The time difference term $|t_{i}^{(\Tr,k)} - t_{i-1}^{(\Tr,k)}|$ ensures that the empirical mean-squared error loss \eqref{eq:training_loss1} matches the $\ell_2$ loss for function estimation. This term will be important to balance the fitting errors among time intervals with different time steps. 
This term would be necessary for the proposed scheme with adaptive (non-uniform) time steps. In our numerical examples, we also performed an ablation study regarding this term 
to demonstrate its necessity; see Figure \ref{fig-Hawkes-const-deltat} in Appendix \ref{sec:add-exp-results} for an example.

\section{Method}\label{sec:method}
 
We state the neural ODE model used for the hidden dynamic in Section \ref{sec:model}. 
The algorithm for adaptive steps  is introduced in Section \ref{sec:adaptive},
and the computational complexity is explained in Section \ref{subsec:comp-cost}.
More implementation details, such as evaluation metrics and choice of thresholds, are given in Appendix \ref{sec:detail}.

\subsection{Neural ODE for RNN model}\label{sec:model}

To be able to model the observation $x(t)$ as a function of a hidden value $h(t)$, we follow the previous continuous-time RNN neural-ODE approach
\cite{chang2019antisymmetricrnn,erichson2020lipschitz} 
to model the hidden dynamics of $h(t)$  as
\begin{equation}\label{eq:ode}
h'(t) = f( h(t), x(t); \theta_h),   
\end{equation}
where $f$ is a neural network parameterized by $\theta_h$. 
If one directly adopts the neural ODE model \cite{chen2018neural} to the hidden state $h(t)$, the ODE model would be   
$  h'(t) = f(h(t),t;\theta)$
without the observed time series data $x(t)$. 
In contrast, the model \eqref{eq:ode} incorporates the observed incoming time series data $x(t)$ as an input to $f$, which is important for modeling the time series data especially when the underlying dynamics is non-stationary. 
The time evolution of the observed series $x(t)$ is modeled by an output neural network $g$ that maps the hidden value $h(t)$ to $x(t)$ as 
\begin{equation}\label{eq:output}
    \hat{x}(t) = g ( h(t); \theta_d ),
\end{equation}
where $g$ is called the output function parameterized by $\theta_d$.

Given the neural network functions $f$ and $g$ (which generally can adopt any architecture) and the observed time series data $x(t)$, from any initial input $h(0)$, we can numerically solve the RNN neural ODE model \eqref{eq:ode} to obtain the $h$ values at any time $t\in(0,T)$ as $h(t) = h(0) + \int_{0}^t f( h(s), x(s); \theta_h) ds$,
and then predict the value of $x(t)$ by $\hat{x}(t)= g ( h(t); \theta_d )$. 
The neural ODE integration can be solved by existing first-order or higher-order schemes, and the back-propagation can be computed by the adjoint method \cite{pontryagin1987mathematical, chen2018neural}.
If one uses the forward Euler scheme, the discrete-time dynamic of $h(t)$ (after incorporating the time step into the network function $f$) becomes $h_{i+1} = h_i + f(h_i,\theta_i)$, which recovers the structure of Residual networks \cite{rumelhart1985learning,hewamalage2021recurrent}.
In this work, we adopt the forward Euler scheme in experiments due to its better stability than higher-order schemes when the dynamic has steep changes. Our methodology of adaptive time grids  can potentially be extended to higher-order differential schemes.

\subsection{Adaptive Time Steps}\label{sec:adaptive}

We propose to learn a neural-ODE RNN model using adaptive time stamps, and thus the method is called \textit{\modelname}.
The construction of adaptive time steps is summarized in Algorithm \ref{alg_adaptive}.
The intuition behind the proposed algorithm is to assign longer (rough) time intervals during time regions where the time series is slowly time-varying (such as ``flat'' curves), while assigning shorter (fine) time intervals during those regions with ``spikes'' (highly non-stationary and fast time-varying regimes). 
For constructing the adaptive time stamps, we assume the initial time grid is sufficiently fine and adopt a dyadic-partition type algorithm to be detailed as follows. 

Given a raw (discrete-time) training window $x(t_0),x(t_1),\dots,x(t_{N})$ sampled at the finest level of the time stamps $0\leq t_0<\cdots< t_N\leq T$.  
For simplicity, below we write it as $x_0,x_1,x_2,\ldots,x_N$.
Without loss of generality, we assume $N$ is a power of two. 
We first define a monitor function $M(\cdot)$ that measures the variation of the sub-sequence $\{x_i,\ldots,x_j\}$, $i<j$. In this paper, we mainly adopt the {\it maximum variation} defined as
\begin{equation}\label{eq:variation}
    M(\{x_i,\ldots,x_j\}) := \max_{i+1 \leq k \leq j} \frac{\| x_{k} - x_{k-1} \|_2}{|t_k-t_{k-1}|} ,
\end{equation}
which captures the maximum variation among any two adjacent time stamps. Here we may also choose $\ell_p$ norms for any $p\geq 1$. 

\begin{figure}[t]
    \centering    
    % \vspace{-4em}
    \includegraphics[width=0.5\linewidth]{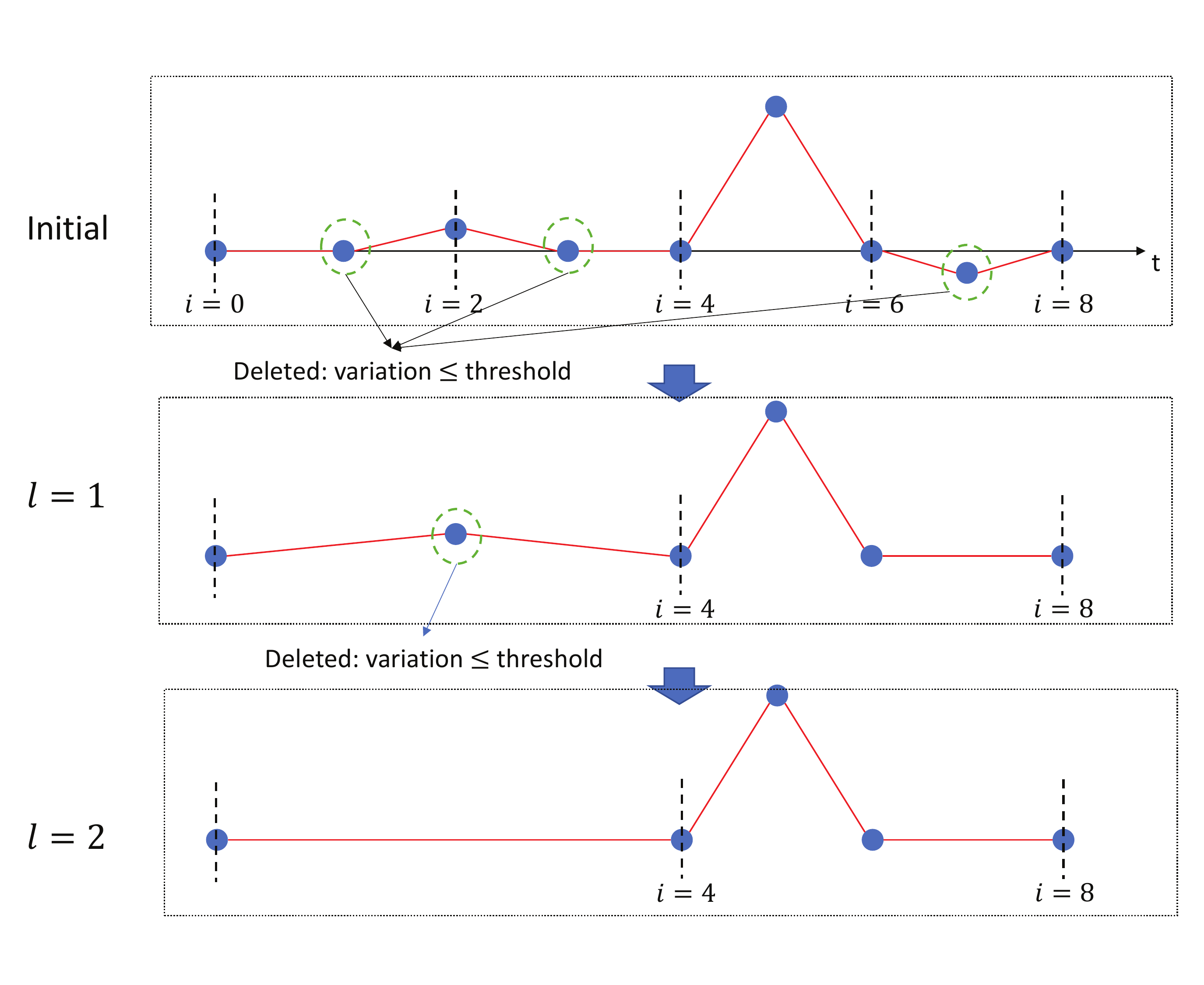}
    % \vspace{-1em}
    \caption{Illustration of adaptive time steps resulted from Algorithm 
    \ref{alg_adaptive}. In this example, $N=8$ and $L=2$; 
    three samples are removed in phase $l=1$, and one sample is removed in phase $l=2$.}
    \label{fig:adaptive_demo}

\end{figure}

We then screen from the finest level of time grids and adaptively merge neighboring time grids if their maximum variation is below a pre-specified threshold $\epsilon>0$. 
In detail, for the first level $l=1$, we group the original $N$ time intervals into $N/2$ sub-intervals (as demonstrated in Figure \ref{fig:adaptive_demo}) and each sub-interval contains three consecutive time stamps: $\{x_0,x_1,x_2\}, \{x_2,x_3,x_4\}, \ldots, \{x_{N-2},x_{N-1},x_N\}$. Then we calculate the maximum variation $M(x_0,x_1,x_2), \ldots, M(x_{N-2},x_{N-1},x_N)$ for each sub-interval. We 
then merge the two consecutive time intervals into one, i.e., remove the middle time stamp $x_{2n+1}$, for $n=0,1,\ldots,N/2-1$, if
\[\label{eq:flag}
M(x_{2n},x_{2n+1},x_{2n+2}) < \epsilon.
\]
In other words, we only keep the time stamps on which the maximum variation exceeds $\epsilon$. 

\begin{algorithm}[t]
\caption{
A dyadic algorithm for selecting adaptive time steps.
}
\label{alg_adaptive}
\begin{algorithmic}[1]
   \STATE {\bfseries Input:} Data series $\{x_0,x_1,x_2,\ldots,x_N\}$; threshold $\epsilon>0$; $L\in \mathbb{Z}_{+}$.  
   \STATE {\bf Initialize}:  $\mathcal D = \emptyset$. A flag vector 
  $\mathtt{Flag}=\{0,0,\ldots,0\}$ of length $N$.
   \FOR{$l=1$ {\bfseries to} $L$}
   \STATE Define a new flag: $\mathtt{Flag_{new}}=\{0,0,\ldots,0\}$ of length $\lfloor N/2^l\rfloor$.
   \FOR{$i = 1$ {\bfseries to} $ \lfloor N/2^l \rfloor$}
   \IF{$\mathtt{Flag}[2(i-1)+1]=\mathtt{Flag}[2i]=0$}
   \STATE Compute the monitoring function $M(\{x_{2^l(i-1)},x_{2^l(i-1)+2^{l-1}},x_{2^li}\})$. 
   % \STATE $\rhd\;$ e.g. the maximum variation defined in \eqref{eq:variation} 
   \IF{$M < \epsilon$}
   \STATE %Merge $\{x_{2^l(i-1)},x_{2^l(i-1)+2^{l-1}},x_{2^li}\}$ into one time step: 
   $\mathcal D = \mathcal D \cup (2^l(i-1)+2^{l-1}).$
   \ELSE
   \STATE Mark $\mathtt{Flag_{new}}[i]=1$.
   \ENDIF
   \ELSE
   \STATE Mark $\mathtt{Flag_{new}}[i]=1$.
   \ENDIF
   \ENDFOR
   \STATE Update $\mathtt{Flag} = \mathtt{Flag_{new}}$.
    \ENDFOR   
   \STATE {\bfseries Output:}  Indexes of removed time steps $\mathcal D$. 
\end{algorithmic}
\end{algorithm}

The above selection procedure is repeated similarly for $l=2,3,\ldots$ until a pre-specified maximum integer $L$. 
The value $L$ corresponds to the roughest time interval. The algorithm is detailed in Algorithm \ref{alg_adaptive}, in which we maintain a set $\mathcal D$ that characterizes which time stamps to be removed. Meanwhile, we also keep a $\mathtt{Flag}$ vector in each round as an indicator of whether the midpoint time stamp was removed in the {\it last} round and $\mathtt{Flag_{new}}$ indicating whether the middle time stamp will be removed in the {\it current} round. 
The elements in the $\mathtt{Flag}$ vector equals $0$ for intervals with slight variation ($M(\cdot)\leq \epsilon$) and $1$ otherwise. 
The primary usage of such $\mathtt{Flag}$ vector is that for two consecutive intervals in round $l'$, we only merge the intervals if both of them are slow-varying intervals (i.e., merged from a previous round $l<l'$).

The output of Algorithm \ref{alg_adaptive} is the set $\mathcal D$ of time stamps to be removed. The model is trained on the remaining time steps only. We provide an illustration in Figure \ref{fig:adaptive_demo} of the algorithm for selecting adaptive time steps. From the final results in Figure \ref{fig:adaptive_demo}, it can be seen that the output of the adaptive time steps uses longer time steps to model stationary periods (from $i=0$ to 4 and from $i=6$ to 8), and uses {\it shorter} time steps to model {\it spikes} (from $i=4$ to 6) in the sequence.

\subsection{Computational Complexity}\label{subsec:comp-cost}
The computational complexity of applying Algorithm \ref{alg_adaptive} in the preprocessing stage to $K^{(\Tr)}$ training windows is $O(K^{(\Tr)}ND)$, where $D$ is the data dimension. For the neural ODE model described as in \eqref{eq:ode}-\eqref{eq:output}, when $f$ possesses the same network structure as a vanilla RNN with $d_{h}$ hidden units and $g$ is a one-layer fully connected network, the complexity in the training process is $O(n_{e}K^{(\Tr)}\bar{N}_{a}d_{h}(d_{h}+D))$, where $n_{e}$ and $\bar{N}_{a}$ represent the number of training epochs and the average length of the adaptive windows, respectively.

Since the computational cost in the training process usually dominates that in the preprocessing step (which happens as long as $n_{e}d_{h} \geq 2^L$), the overall complexity of the \textit{\modelname} model is $O(n_{e}K^{(\Tr)}\bar{N}_{a}d_{h}(d_{h}+D))$. This is of the same order as the complexity of training a vanilla RNN (we refer to Appendix \ref{sec:net-struct} for the specific structure) with $d_{h}$ hidden units in $n_{e}$ epochs, using $K^{(\Tr)}$ training windows with the same length $\bar{N}_{a}$. 
Therefore, compared with the complexity when training with the original finest $N$ time grids,
the complexity associated with the adaptive method will be reduced by a factor of {$\bar{N}_{a}/N$.} {The smallest achievable complexity will be reduced by a factor of $1/2^L$ when choosing a sufficiently large threshold $\epsilon$. }

\section{Theory}\label{sec:theory}

In this section, we provide the recovery consistency of the training objective (Section \ref{subsec:theory-consistency}) and approximation error guarantee of the RNN-ODE model revealing the benefit of adaptive step size (Section \ref{subsec:approx-analysis}).
All proofs can be found in Appendix \ref{app:proofs}.

\subsection{Function Estimation for Event-type Data}\label{subsec:theory-consistency}

We present the theoretical analysis
for function estimation based on the proposed model under counting-type time series. 
It is worthwhile noting that counting-type time series represent a special class of continuous-time models since they exemplify the {\it extreme} case of ``spike-like'' data, where we have discontinuities from zero to one, as shown at the right end of Figure\,\ref{fig:spikes}.

For event-type sequences, the raw data contains a list of event times $0<t_1<t_2<\ldots<t_n<T$ 
on the time horizon $[0,T]$. Each timestamp %$t_i$ 
is the time when an event happens. 
In practice, the estimation is performed on {\it discrete-time grids}.
Define the counting process $N(t):= \sum_{i=1}^n \mathbf{1}( t_i \leq t)$ as the total number of events happened before time $t$. We convert such continuous-time data into discrete observations by discretizing the time interval $[0,T]$ into $M$ intervals of equal length $\Delta t= T/M$, and then let $x_m = N( m\Delta t) -  N((m-1) \Delta t)$, $m=0,1,\ldots,M$ (by convention $x_0=0$). 
When $\Delta t$ is chosen sufficiently small, it becomes the Bernoulli process where $x_i \in \{0,1\}$. 

We consider the temporal Hawkes processes \cite{rasmussen2011temporal}, in which the values $x_i$ are mostly zero under mild assumptions, corresponding to sparse ``spikes''. Such temporal Hawkes processes can be characterized by its {\it conditional intensity} function defined as
$$
\lambda^*(t) = \lim_{\Delta\rightarrow0}\Delta^{-1}\mathbb{E}[N(t+\Delta)-N(t)|\mathcal F_{t}],
$$
where the filtration $\mathcal F_{t}$ stands for the information available up to time $t$. 
In the case of Hawkes processes, $\lambda(t)= \mu + \alpha \int_0^t \phi(t-s)dN(s)$ is simply a linear function of past jumps of the process,
where $\phi(\cdot)$ is the influence kernel. For example, under the special case of exponential kernels, the intensity function becomes $\lambda^*(t) = \mu + \alpha \beta \int_{0}^{t} e^{-\beta (t-\tau)} dN(\tau)$.

The intensity function recovery consistency by minimizing least-square population loss
is proved in Theorem \ref{prop-error}
under a memory constraint.
We parameterize the function by a neural network (NN) based structure characterized as in \eqref{eq:ode}-\eqref{eq:output}. We define the prototypical network architecture below.

\begin{definition}\label{def:RNN-ODE}
Define the function class 
$\textit{NN-ODE}(d_{\mathrm{out}},L_h,p_h,L_d,p_d)$ as 
\begin{equation}\label{eq:functionODERNN}
\begin{aligned}
&\textit{NN-ODE}(d_{\mathrm{out}},L_h,p_h,L_d,p_d) \coloneqq \{  F: \mathbb R \mapsto \mathbb R^{d_{\mathrm{out}}}| F(t) = g(h(t)),  h'(t) = f(h(t),x(t)), \\
& \hspace{25pt}  g \text{ is NN with $L_d$ layers and max-width $p_d$},  h \text{ is NN with $L_h$ layers and max width $p_h$}.\}
\end{aligned}
\end{equation}
\end{definition}

\begin{theorem}\label{prop-error}
Assume there exist $d$ buffer time steps with samples $x_{-d},\ldots,x_{-1}$ prior to the Hawkes count data $\{x_0,\ldots,x_M\}$ and each time step has duration $\Delta t = T/M$. We further assume $\textit{NN-ODE} (d_{\mathrm{out}}, L_h,p_h, L_d,p_d)$ is rich enough to model the true intensity function.    
Then the minimizer $F^*$ to the 
population loss function
\[
\Psi(F) \coloneqq \sum_{m=1}^M \mathbb E\big[ ( x_m - F(m \Delta t)\Delta t)^2 \big| x_{m-d}\ldots x_{m-1}\big],\]
optimized within the neural network class $F \in \textit{NN-ODE}(D,L_h,p_h,L_d,p_d)$, satisfies $F^*(m\Delta t) = \tilde \lambda(m) := \frac{1}{\Delta t}\int_{(m-1)\Delta t}^{m\Delta t} \lambda^*(t)dt$, which is the discretized intensity. 
\end{theorem}

\begin{remark}\label{rmk:theory}
    We have the recovered intensity function $F^*(m\Delta t) = \tilde \lambda(m)$ and is extendable to the entire time horizon as $F^*(t) = F^*(m\Delta t) \mathbf{1}\{(m-1)\Delta t < t \leq m\Delta t\}$ for any $t\in[0,T]$. 
    In Appendix \ref{app:intensity-proof} we extend the analysis to show that under the asymptotic scenario when $M\rightarrow \infty$%($\Delta t = T/M \rightarrow 0$)
    , we have that $\int_0^T |F^*(t) - \lambda^*(t)|dt \rightarrow 0$. 
\end{remark}

The above argument is made under the population loss, showing that using the least-square loss function can recover the actual intensity function for the Hawkes process. It is mainly due to the generality of the ODE model \eqref{eq:ode} and \eqref{eq:output}, which is consistent with Hawkes process and most time series models.
The argument may be extended to empirical processes by utilizing the empirical concentration of the process. 

\subsection{Approximation Analysis of \modelname}
\label{subsec:approx-analysis}

For theoretical generality, in this subsection, we consider the continuous-time process $y(t)\in\R^{D'}$ satisfying 
\begin{equation}\label{eq:dyn-sys}
h'(t)=f(h(t),x(t)), \quad y(t)=g(h(t)), 
\quad h(0) = h_0, 
\quad t\in [0,T],    
\end{equation}
where $x(t)\in\R^D$ is the observable input data and $h(t)\in\R^{d_h}$ is the underlying hidden process from some initial value $h_0$. 
Taking $y(t) $ to be $x(t)$ reduces the model to the case \eqref{eq:ode}\eqref{eq:output} considered in the other parts of the work. 
We provide two theorems: Theorem \ref{thm:continuous-error} proves the uniform approximation to $y(t)$ by continuous-time RNN-ODE model without time discretization;
Theorem \ref{thm:discret-error} further takes into account the discrete-time scheme and obtains the approximation on a time grid.

\paragraph{Approximation of the continuous-time model.}
We will use neural network functions $f_\theta$ and $g_\phi$ to approximate the functions $f$ and $g$, respectively, see Lemma \ref{lemma:func-approx}.
Given $x(t)$ on $[0,T]$, let $h_{\rm NN}(t)$ be the solution to the hidden-process ODE 
$h_{\text{NN}}'(t)=f_\theta(h_{\text{NN}}(t),x(t))$
from $h_{\rm NN}(0) = h_0$. This leads to the output process $y_{\rm NN}(t)$ defined by
\begin{equation}\label{eq:dyn-sys-approx}
h_{\text{NN}}'(t)=f_\theta(h_{\text{NN}}(t),x(t)), \quad y_{\text{NN}}(t)=g_\phi(h_{\text{NN}}(t)), 
\quad h_{\rm NN}(0) = h_0,
\quad t\in [0,T].
\end{equation}
The approximation of $y_{\rm NN}(t)$ to $y(t)$ will be based on the approximation of $f_\theta$ and $g_\phi$, which calls for the regularity condition of the system \eqref{eq:dyn-sys}. 
We take the following technical conditions.
\begin{assumption}\label{assump1}
\begin{enumerate}[label=(A\arabic*)]%[label={(\Alph*)}]
    \item 
    The observed process $x:[0,T]\to [-1,1]^D$ and is Lipschitz continuous over $t$;
    The hidden process $h:[0,T]\to [-1,1]^{d_h}$.
    \item $f:[-1.1,1.1]^{d_h}\times [-1,1]^D\to \R^{d_h}, (\eta,x)\mapsto f(\eta,x)$, and is Lipschitz continuous with respect to both $\eta$ and $x$.
    \item $g:[-1.1,1.1]^{d_h}\to[-1,1]^{D'}, \eta\mapsto g(\eta)$ is Lipschitz continuous.
\end{enumerate}    
\end{assumption}

We let $L_g$ denote the global Lipschitz constant of $g$ on $[-1.1,1.1]^{d_h}$.  For $f$, both global and local Lipschitz constants on the domain $[-1.1,1.1]^{d_h}\times [-1,1]^D$ are used. More detailed definitions of these constants will be introduced in Lemma \ref{lemma:func-approx} (for the global constant) and Theorem \ref{thm:continuous-error} (for the local constant).

The next lemma directly follows by applying \cite{yarotsky2017error} to the case where $f$ and $g$ have 1st-order regularity (Lipschitz continuity). The proof is given in appendix \ref{app:approx-proof}. 

\begin{lemma}\label{lemma:func-approx}
For any $\epsilon_f, \epsilon_g > 0$, 
there exist neural networks $f_\theta, g_\phi$ such that 
\begin{equation}\label{eq:func-approx}
    \max_{\eta\in [-1.1,1.1]^{d_h}, x\in [-1,1]^D} \left\|f(\eta,x) - f_\theta(\eta,x)\right\|_2 < \epsilon_f, \max_{\eta\in [-1.1,1.1]^{d_h}} \left\|g(\eta) - g_\phi(\eta)\right\|_2 < \epsilon_g,
\end{equation}
and 
\begin{itemize}
    \item $f_\theta$ has $O(\ln\frac{C_f}{\epsilon_f}+\ln d_h+1)$ layers and $ O( ({C_f}/{\epsilon_f})^{{d_h+D}}(\ln\frac{C_f}{\epsilon_f}+\ln d_h+1))$ trainable parameters.
    \item $g_\phi$ has $O(\ln\frac{C_g}{\epsilon_g}+\ln D'+1)$ layers and $O(({C_g}/{\epsilon_g})^{d_h}(\ln\frac{C_g}{\epsilon_g}+\ln D'+1))$ trainable parameters.
\end{itemize}
The constants in big-$O$ may depend on $D, D'$, and $d_h$. Here $C_f \coloneqq \max \{ L^{f,h}, L^{f,x}, M_f\}$, where $M_f= \sup_{(\eta,x)\in [-1.1,1.1]^{d_h}\times[-1,1]^D} \|f(\eta,x)\|$ and $L^{f,h}, L^{f,x}$ are denote the Lipschitz constant of $f$ on $[-1.1,1.1]^{d_h}\times [-1,1]^D$ (see formal definitions in \eqref{eq:def-L^fh} in the proof of Lemma \ref{lemma:func-approx} in Appendix \ref{app:approx-proof}). $C_g\coloneqq \max\{L_g,M_g\}$, and $M_g= \sup_{\eta\in [-1.1,1.1]^{d_h}} \|g(\eta)\|$.
\end{lemma}

For the spike-like data, the majority of the regions are slow-varying, with the spikes occupying only a minor part of the whole interval $[0,T]$. Thus, the whole interval $[0,T]$ may be partitioned into two disjoint sets $\mathcal{D}_1$ and $\mathcal{D}_2$, each of which consisting of unions of disjoint intervals in $[0,T]$. To characterize this partition more precisely, we define the constants related to an interval in $[0,T]$ as follows:

\begin{itemize}
\item For an interval $[s,t]\subset [0,T]$, we define the domains $B^h, B^x$ as
\begin{equation}\label{eq:def-local-balls}
B^h: = (h([s,t]) + B_r^{d_h}) \subset [-1.1, 1.1]^{d_h},
\quad B^x: = (x([s,t]) + B_r^{D}) \cap [-1, 1]^{D},    
\end{equation}
with $r=0.1$, and $B_r^{d_h}, B_r^{D}$ represent balls with radius $r$ in $\R^{d_h}, \R^{D}$ respectively (see Figure \ref{fig:x_h_demonstrate} for illustration). Here, $h([s,t]) + B_r^{d_h}$ means the Minkowski addition, namely $\{h_1+h_2, h_1\in h([s,t]), h_2\in B_r^{d_h}\}$, and $x([s,t]) + B_r^{D}$ is defined in the same way. 
Then, we denote
\begin{align}\label{eq:def-local-Lipschitz}
L_{[s,t]}^{f,h} &\coloneqq \sup_{x\in B^x}\sup_{\eta_1,\eta_2\in B^h} \frac{\|f(\eta_1,x)-f(\eta_2,x)\|}{\|\eta_1-\eta_2\|},   \nonumber\\ L_{[s,t]}^{f,x} &\coloneqq \sup_{h\in B^h}\sup_{x_1,x_2\in B^x} \frac{\|f(\eta,x_1)-f(\eta,x_2)\|}{\|x_1-x_2\|}, 
\end{align}
as the local Lipschitz constants of $f$ within the domain $B^h\times B^x$, and
\begin{equation}\label{eq:def-local-max-f}
M^f_{[s,t]} \coloneqq \sup_{( \eta, x )\in B^h \times  B^x} \|f(\eta,x)\|_2.    
\end{equation} 

\end{itemize}

\begin{figure}[t]
    \centering    
    % \vspace{-4em}
    \includegraphics[width=0.75\linewidth]{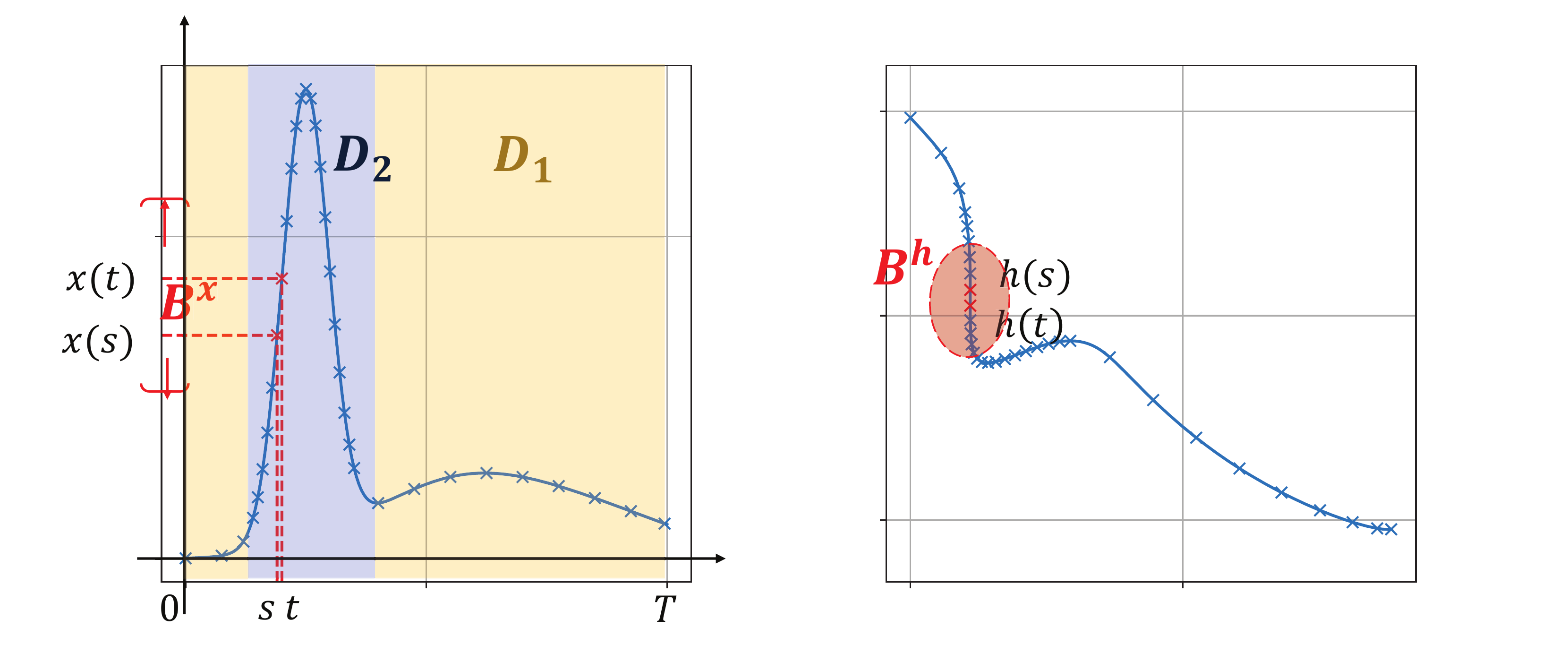}
    % \vspace{-1em}
    \caption{Demonstration of the domains $B^x$ and $B^h$ defined as in \eqref{eq:def-local-balls} for the time interval $[s,t]$ (here $d_h=2, D=1$). The domains $\mathcal{D}_1$ and $\mathcal{D}_2$ that correspond to slowly and fast varying regions are colored in orange and blue respectively.
    % \xc{the definition on $[s_{i-1}, s_{i}]$ is similarly }
    }
    \label{fig:x_h_demonstrate}

\end{figure}
With the local Lipschitz constant defined as above, we suppose that any time grid $[s_1,t_1]$ in $\mathcal{D}_1$ corresponds to a local Lipschitz constant $L_{[s_1,t_1]}^{f,h}\leq L_{\rm low}$. On contrast, if $[s_2,t_2]$ belongs to $\mathcal{D}_2$, the local Lipschitz constant $L_{\rm low}<L_{[s_2,t_2]}^{f,h}\leq L_{\rm high} (\leq L^{f,h})$. Here, $\mathcal{D}_1$ is comprised of regions with slow variations, while $\mathcal{D}_2$ encompasses regions with sharp changes, as demonstrated in Figure \ref{fig:x_h_demonstrate}. It may often be the case that $|D_1|$ is greater than $|D_2|$. Then, we define
\begin{equation}\label{eq:def-L-avg}
L^{\rm (avg)}\coloneqq \frac{1}{T} ( L_{\rm low} |D_1| + L_{\rm high} |D_2|).    
\end{equation}
Following Lemma \ref{lemma:func-approx} and the partition described above, Theorem \ref{thm:continuous-error} below provides the approximation results for the continuous-time process $y(t)$ using \eqref{eq:dyn-sys-approx}.

\begin{theorem}\label{thm:continuous-error} 
Under Assumption \ref{assump1} and for $L^{(\rm avg)}$ defined as in \eqref{eq:def-L-avg}, suppose $\epsilon_f, \epsilon_g>0$ and $\epsilon_f$ satisfies
% \xc{this condition is only about epsf right? why epsg need to satisfy?}
\begin{equation} 
Te^{L^{(\rm avg)}T} \epsilon_f < 0.1, 
\end{equation}
and let $f_\theta, g_\phi$ 
be the neural networks satisfying \eqref{eq:func-approx} (the model complexity is bounded as in Lemma \ref{lemma:func-approx}), then
\begin{equation}\label{eq:cts-err-analysis}
\max_{t\in[0,T]} \|y(t)-y_{\text{NN}}(t)\| < \epsilon_g + L_g Te^{L^{(\rm avg)} T} \epsilon_f.
\end{equation}
\end{theorem}

\begin{remark}[Interpretation of $L^{(\rm avg)}$ and local Lipschitz constants]

\eqref{eq:cts-err-analysis} can provide an improved bound because when the data have sharp changes, $L_{\rm high}$ (as the $\infty$-norm of the Lipschitz constant over time) can be large while $L^{\rm (avg)} = ( L_{\rm low} |D_1| + L_{\rm high} |D_2|)/T$ (as certain $L^1$-norm of the Lipsthictz constant over time) may stay at a smaller value. The partition $\mathcal{D}_1\cup\mathcal{D}_2$ reflects how adaptively choosing grids may help improve the theoretical results, and this will be further explored in the next subsection. Therein $x(t)$ will only be observed at a discrete time grid, which can be adaptively chosen according to the local Lipschitz constants (see Theorem \ref{thm:discret-error} for more details).

\end{remark}

\begin{remark}[Arbitrary desired accuracy in \eqref{eq:cts-err-analysis}]
For any $\varepsilon > 0$, we can choose 
\[
\epsilon_f < \frac{1}{T\exp(L^{(\rm avg)}T) } \min\{0.1, \frac{\varepsilon}{2L_g}\},
\quad \epsilon_g < \frac{\varepsilon}{2},
\]
then the right-hand side of \eqref{eq:cts-err-analysis} is bounded by $\varepsilon$.    
\end{remark}

\paragraph{Approximation under time discretization.}

We assume that $x(t)$ is only observed at discrete time grids $\{ t_i\}_{i=1}^N$ instead of on the whole interval $[0,T]$. Below, the time grids can be chosen adaptively, which will be detailed in Remark \ref{rmk:discrete-err}. Given the time grids $\{ t_i\}_{i=1}^N$, we define the following constants that will be used in the theorems later:

\begin{itemize}
\item By (A2), for each $i$, let $L^{f,h}_i, L^{f,x}_i$ and $M^f_i$ be defined as in \eqref{eq:def-local-Lipschitz} and \eqref{eq:def-local-max-f} respectively, where we take the interval $[s,t]$ as $[t_{i-1}, t_i]$.
\item By (A1), for each $i$, let $L^x_i$ be the Lipschitz constant of $x(t)$ on $t\in [t_{i-1},t_i]$. $i=1,\dots,N+1$ (we follow the convention that $t_0 = 0, t_{N+1}=T$). 
\end{itemize}
For $\Delta t_i\coloneqq t_i - t_{i-1}$ and $\hat{h}_{\text{NN}}(0)=h_0$, the forward Euler scheme is applied on \eqref{eq:dyn-sys-approx} as follows: 
\begin{equation}\label{eq:scheme}
\hat{h}_{\text{NN}}(t_i) = \hat{h}_{\text{NN}}(t_{i-1}) + \Delta t_i f_\theta(\hat{h}_{\text{NN}}(t_{i-1}),x(t_{i-1})), \quad \hat{y}_{\text{NN}}(t_i) = g_\phi(\hat{h}_{\text{NN}}(t_i)),\quad i=1,\dots,N.    
\end{equation}

Compared to Theorem \ref{thm:continuous-error}, Theorem \ref{thm:discret-error} below additionally accounts for the discretization error from the numerical integration, providing an upper bound of the approximation error using $x(t)$ observed at discrete time grids. Theorem \ref{thm:discret-error} focuses on the forward Euler method, and we refer to Remark \ref{rmk:extension-high-oder} for its extension to higher-order schemes.

\begin{theorem}\label{thm:discret-error} 
Under Assumption \ref{assump1} and given a time grid $\{ t_i\}_{i=1}^N$ on $[0,T]$ at which $x(t)$ is observed.
Suppose $\epsilon_f, \epsilon_g>0$ and $\epsilon_f, \Delta t_j$ satisfy
\begin{equation}\label{eq:condition-thm}
T\exp(\sum_{i=1}^N L^{f,h}_i\Delta t_i)  \left(\epsilon_f + \max_{j}\{\mu_j\Delta t_j\}\right) < 0.1, 
\end{equation}
where
\[\mu_j\coloneqq L^{f,h}_jM_j^f + L^{f,x}_jL^x_j,\]
and let $f_\theta, g_\phi$ be the neural networks satisfying \eqref{eq:func-approx} (the model complexity is bounded as in Lemma \ref{lemma:func-approx}), 
then
\begin{equation}\label{eq:discrete-err-analysis}
\max_i\|y(t_i)-\hat{y}_{\text{NN}}(t_i)\|\leq \epsilon_g + L_gT\exp(\sum_{i=1}^N L^{f,h}_i\Delta t_i)  \left(\epsilon_f + \max_{j}\{\mu_{j}\Delta t_j\}\right).
\end{equation}
\end{theorem}

The condition \eqref{eq:condition-thm} in Theorem \ref{thm:discret-error} is imposed to guarantee that the numerically integrated hidden states $\{\hat{h}_{\text{NN}}(t_i)\}$ belong to $[-1.1,1.1]^{d_h}$, so that the approximation results in Lemma \ref{lemma:func-approx} are applicable.

\begin{remark}[Arbitrary desired accuracy in \eqref{eq:discrete-err-analysis}]\label{rmk:discrete-err}
For any $\varepsilon>0$, suppose the time grids satisfy that
\begin{equation}\label{eq:rmk-cond-grids}
\max_j\{\mu_j\Delta t_j\}\}
<  \frac{1}{ T \exp(\sum_{i=1}^N L^{f,h}_i\Delta t_i)}
\min\{0.05, \;\frac{\varepsilon}
{3L_g }\},    
\end{equation}
then we can choose 
\[
\epsilon_f < \frac{1}{ T \exp(\sum_{i=1}^N L^{f,h}_i\Delta t_i)}
\min\{0.05, \;\frac{\varepsilon}
{3L_g }\},\quad 
\epsilon_g < \frac{\varepsilon}{3},
\]
to make the right-hand side of \eqref{eq:discrete-err-analysis} bounded by $\varepsilon$. 
\end{remark}

\begin{remark}[Extension to higher-order integration schemes]\label{rmk:extension-high-oder}

% Note that to fit the value of $y(t)$ at the current time $t_i$, the scheme \eqref{eq:scheme} only utilizes the historical part of the observable data $x(t)$. 
%
% Moreover, 
The numerical integration scheme \eqref{eq:scheme} can be extended to the multi-step explicit methods of higher orders (e.g. Runge-Kutta methods), given that the time grid selection appropriately fulfills the requirements of the integration scheme. For example, we may choose $t_{i+1}-t_i = t_i-t_{i-1}$ for adjacent sub-intervals $[t_{i-1},t_i], [t_i,t_{i+1}]$ to apply the commonly used RK4 method .

\end{remark}

Theorem \ref{thm:discret-error} provides insights into the utility of the adaptive steps for improving the model fitting performance, which is reflected in the last term, involving $\max_{j}\{\mu_j\Delta t_j\}$, in Eq. \eqref{eq:discrete-err-analysis}. Specifically, time grids may be selected such that $\Delta t_i$ is small if $L^x_i$ is great, indicating a steep change in $x(t)$ for $t\in[t_{i-1},t_i]$. 
On the contrary, when the variation in $x(t)$ is smaller, we employ larger $\Delta t_i$ to reduce the total number of required time grids.

\section{Numerical Experiments}
\label{sec:numerical}

We validate the performance of the proposed method using three types of datasets: the simulated spiral series, the simulated event data, and a real ECG dataset. We present the complexity vs. accuracy tradeoff curves for different methods, and demonstrate the advantage of the \textit{\modelname} method. 
The code can be found at 
\href{https://github.com/Yixuan-Tan/RNN_ODE_Adap}{\texttt{https://github.com/Yixuan-Tan/RNN\_ODE\_Adap}}.

\subsection{Trained Models}

In this section, we examine and report the performance of two models,  \textit{RNN-ODE} and \textit{\modelname}.
Both models are trained to minimize  the loss function as in \eqref{eq:training_loss1}, and the difference lies in the choice of the time grid. 
Specifically, \textit{RNN-ODE} is trained using regular (non-adaptive) time steps, and \textit{\modelname} is trained with adaptively selected time steps by Algorithm \ref{alg_adaptive}. 
The architecture of both models is the same as vanilla RNN; see more details in Appendix \ref{sec:net-struct}.

As explained in Section \ref{subsec:comp-cost}, the computational cost of training the models is proportional to the average length of the training windows. Hence, in this section, when we compare the performance under varying complexities, the ``complexities'' are discussed in terms of averaged ``numbers of grids'' of the training data.
More details of the experiment settings are in Appendix \ref{sec:append-exp}, with the boxplots of the error plots and additional results provided in Appendix \ref{sec:add-exp-results}.

\begin{figure}[t]
\centering

\includegraphics[width=0.5\linewidth]{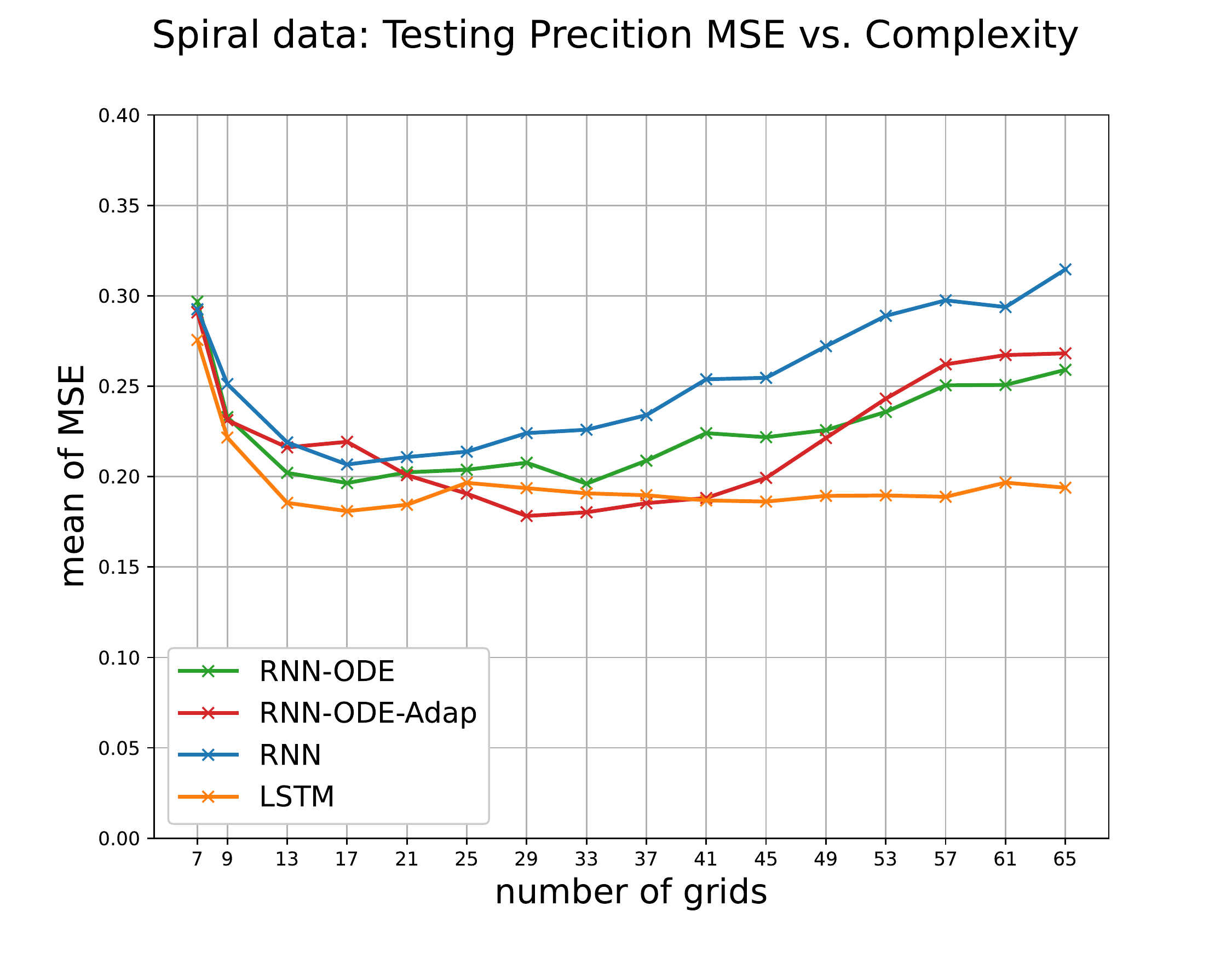}

\caption{
Comparison of the MSE prediction errors on the simulated spiral data generated from Eq. \eqref{eq:spiral-ODE} for RNN, LSTM, RNN-ODE, RNN-ODE-Adap. The $x$-axis represents the average length of the training windows, which reflects the complexity of the models (see Section \ref{subsec:comp-cost}). 
}
\label{fig:spiral-error}
\end{figure}

\subsection{Simulated Spiral Data}\label{sec:spiral}

We first investigate the capability of our method % (denoted as RNN-ODE) 
to fit and capture the underlying dynamics of the simulated spiral data. For a given matrix $A\in\mathbb{R}^{2\times 2}$, one spiral is generated by integrating the ODE 
\begin{equation}
x'(t)= f(x(t)) = \|x(t)\|^{-2} A x(t), \label{eq:spiral-ODE}    
\end{equation}
over the time span $[0,T]$, with the initial value $x(0) = x_0\in\mathbb{R}^2$.
The initial training and testing windows are of length $N=64$, corresponding to the largest complexity shown in Figure \ref{fig:spiral-error}. 

We first compare the on-sample prediction performance with RNN and LSTM \cite{hochreiter1997long} under different computational complexities. 
For each testing window, we use the first half as available historical data and perform predictions for the second half. Figure \ref{fig:spiral-error} shows the averaged Mean Squared Errors (MSEs) computed as in Eq. \eqref{eq:pred-error} for the models, with varying complexities.
From the trade-off curves  in Figure \ref{fig:spiral-error}, we observe that when the training cost is relatively low, namely the training windows are short, all the models underfit and do not learn the dynamics well, resulting in poor prediction performance. When the training cost increases and the training windows consist of more time steps, the models overfit, and prediction accuracy worsens.

\begin{figure}[t]
\centering
\begin{tabular}{cc}
\includegraphics[width=0.98\linewidth]{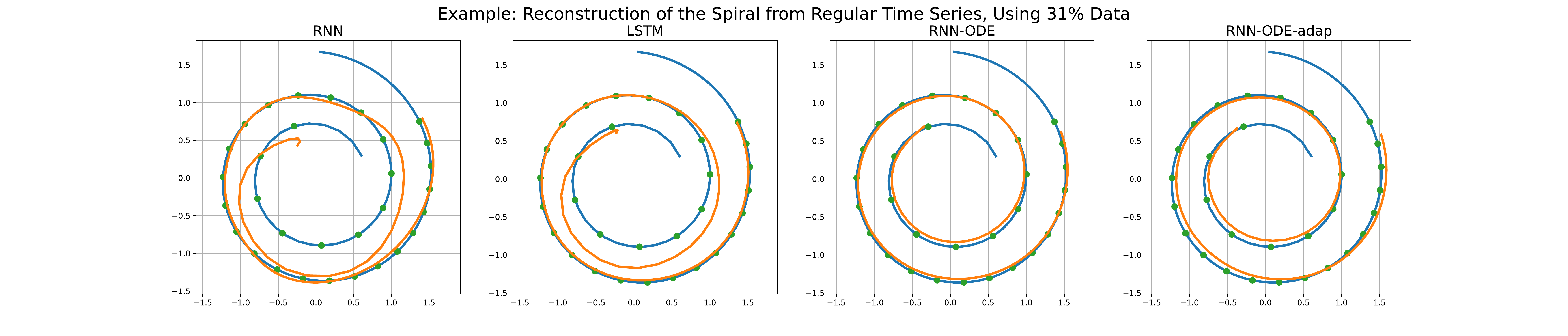}\\
\includegraphics[width=0.98\linewidth]{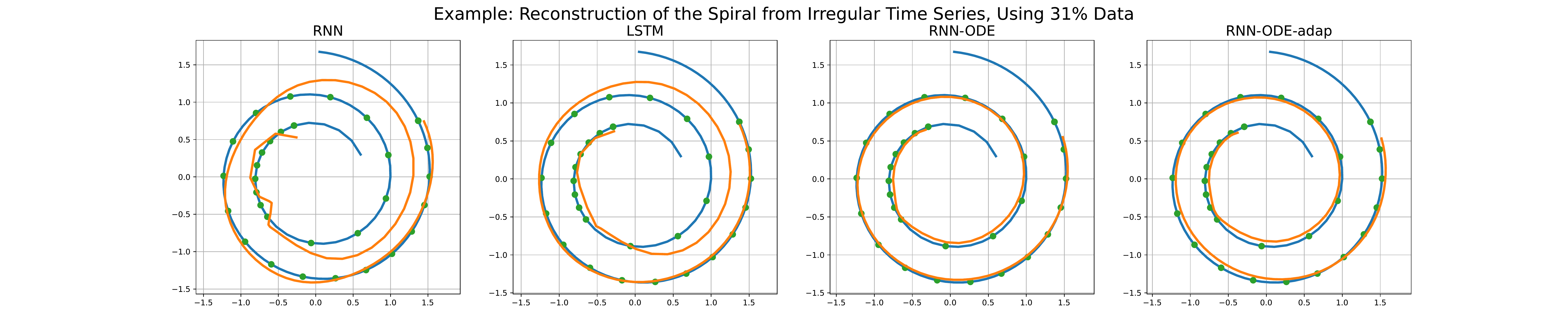}
\end{tabular}
\caption{
Comparison of RNN, LSTM, RNN-ODE, RNN-ODE-Adap on the reconstruction of simulated spiral data generated from Eq. \eqref{eq:spiral-ODE}, using regular (upper) v.s. irregular (lower) time series.}
\label{fig:spiral-reconstruct}
\end{figure}

Moreover, we observe from Figure \ref{fig:spiral-error} that for the same complexity, RNN-ODE significantly improves the forecasting performance compared to the vanilla RNN, especially when the number of grids is not too small so that the models begin to learn the dynamics well. Additionally, {\modelname} further achieves smaller prediction errors than RNN-ODE since it selects data points more informatively with the same number of grids. Finally, we note that while LSTM performs best in most cases, it possesses a more complex network structure. We refer to Appendix \ref{sec:add-exp-results} for additional results about the LSTM and Lipschitz-RNN \cite{erichson2020lipschitz} variants of the adaptive model.

Figure \ref{fig:spiral-reconstruct} shows examples of spiral reconstructions using about 30\% of the data. The time steps might be obtained by interpolation and regular (the upper half of Figure \ref{fig:spiral-reconstruct}), or be chosen adaptively by Algorithm \ref{alg_adaptive} and thus irregular (the lower half of Figure \ref{fig:spiral-reconstruct}). 
We can see that there are mismatches between shapes reconstructed by RNN and LSTM and the ground truth spiral shape. In contrast, we note that RNN-ODE and {\modelname} are consistent with the underlying spirals.

\subsection{Simulated Point-Process Data} \label{sec:hawkes}

We further apply our method to a simulated example of event times data generated from temporal Hawkes processes \cite{rasmussen2011temporal} as described in Section \ref{sec:theory}. 
We train the model \eqref{eq:ode}-\eqref{eq:output} and estimate the true intensity function $\lambda(t)$ using the output $x(t)$. The mean squared loss $L = \int_0^T (dN(t)/d t - \lambda(t) )^2 dt$ is used when fitting the neural ODE model.

The fitting errors of the four models versus the complexity are shown below in Figure \ref{fig:Hawkes-error} (left). It can be observed that for all the models, the fitting errors decrease as the training complexity increases. Furthermore, {\modelname} achieves the best fitting performance. The right panel shows the log-log plot of RNN-ODE and {\modelname}, from which we can see more clearly that for fixed model complexity (network structure), the proposed model approaches the true intensity function.

Figure \ref{fig:Hawkes-error} (right) shows two examples of fitting performance. In this example, all models use 33 grids on average. Thus, the complexity is 50\% of the largest one. It can be observed that RNN fails to capture the smooth decay of the kernel. Furthermore, we can see that 
with the adaptive choice of time steps which is more refined, {\modelname} can learn the dynamics of the intensity function much better -- it can estimate the ``jump'' in the intensity accurately.

\subsection{Real Data: ECG Time Series}\label{sec:ecg}

We validate the proposed {\modelname} on one public electrocardiography (ECG) dataset PTB-XL \cite{wagner2020ptb,goldberger2000physiobank}. 
We focus on learning the underlying dynamics of ECG signals using the RNN-ODE model and use adaptive time steps for ``spikes'' in data series.   
We remark that windows of the highest sampling rate are chosen to have $N=96$ time grids, in which usually two cycles are contained. In this way, the prediction of the second half given the first half would be more meaningful.

Figure \ref{fig:ECG-error} (left) shows the on-sample prediction MSEs of the four methods for two different prediction lengths, 24 and 48, which are $1/4$ and $1/2$ of the whole window. Here, the
prediction is performed with the original finest grids by integrating the ODE function. It can be seen that RNN-ODE has smaller prediction errors than RNN on average, and adding addictive steps helps achieve slightly better performance. Furthermore, LSTM still achieves the smallest error most of the time. The reason for this is explained in Section \ref{sec:spiral} and may be due to its more complex network structure.

\begin{figure}[h]
    \centering
\begin{tabular}{cc}    
\hspace{-1em} \includegraphics[width=0.52\textwidth]{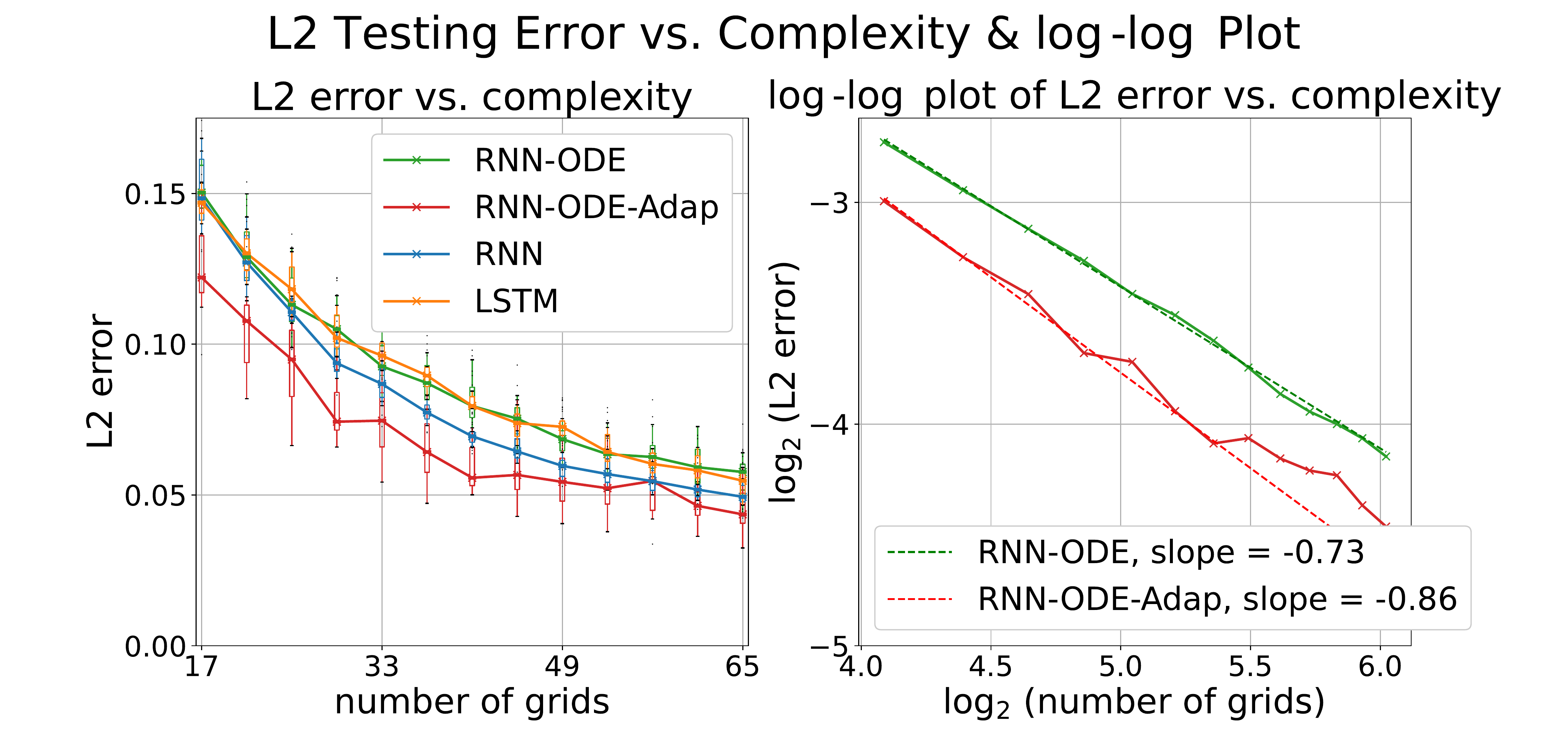}  & \hspace{-4em} \includegraphics[width=0.58\linewidth]{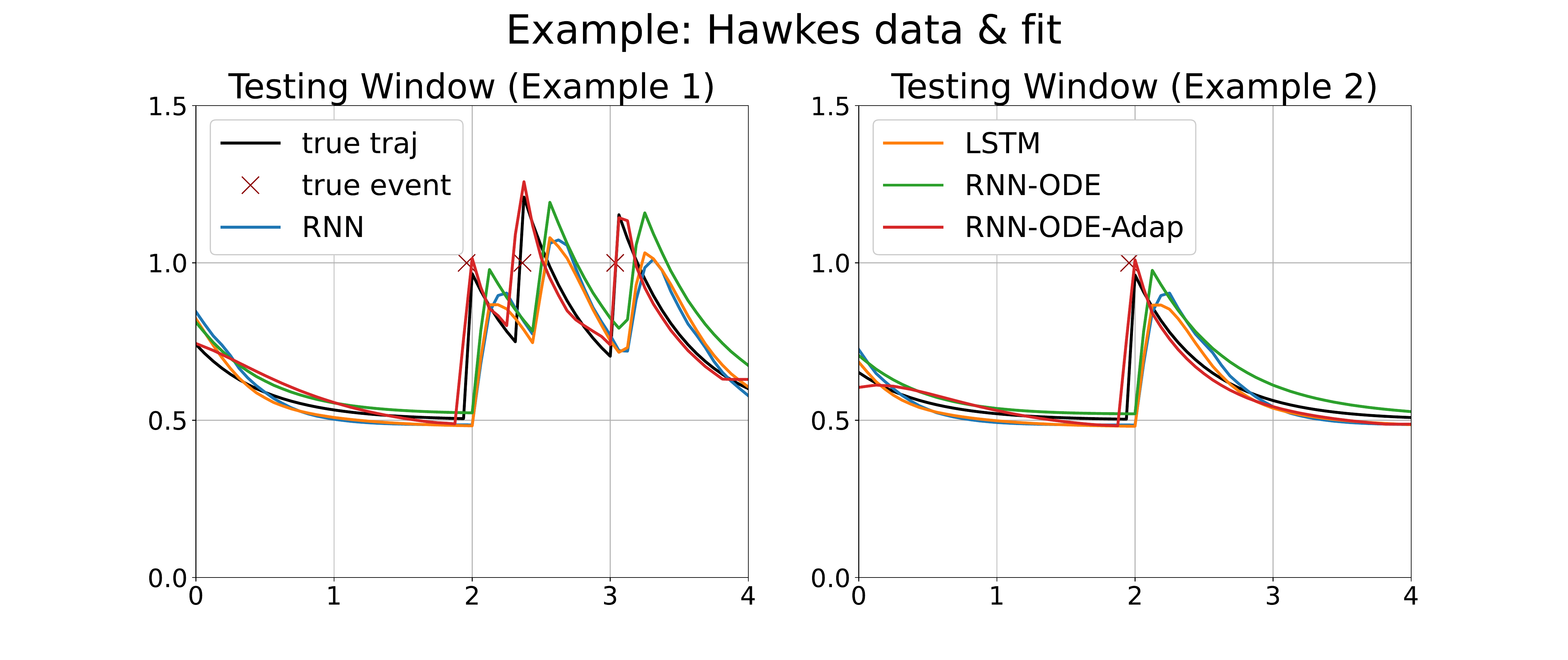} \\
    \end{tabular}
    \caption{Left: Comparison of the fitting errors of the underlying intensity function of the simulated event-type data generated from the Hawkes process for RNN, LSTM, RNN-ODE, RNN-ODE-Adap. $x$-axis represents computational complexity, $y$-axis is the fitting error computed as in Eq. \eqref{eq:fit-error}. 
    Right: Examples of fitted intensity function of the simulated event times data generated from the Hawkes process using RNN, LSTM, RNN-ODE, RNN-ODE-Adap.}
    \label{fig:Hawkes-error}
\end{figure}

Figure \ref{fig:ECG-error} (right) and Figure \ref{fig:ECG-example2} in Appendix \ref{sec:add-exp-results} present examples of prediction on the testing windows for prediction lengths 48 and 24, respectively.  These examples demonstrate that {\modelname} models capture the cycles and trends of the ECG more effectively than RNN.
The good performance implies that the proposed algorithm could be used to fit and predict the ECG-type signal well. This also implies potential future applications of {\modelname} in real healthcare problems.

\begin{figure}[t]
    \centering
    \begin{tabular}{cc}
       \hspace{-1.8em}\includegraphics[width=0.58\textwidth]{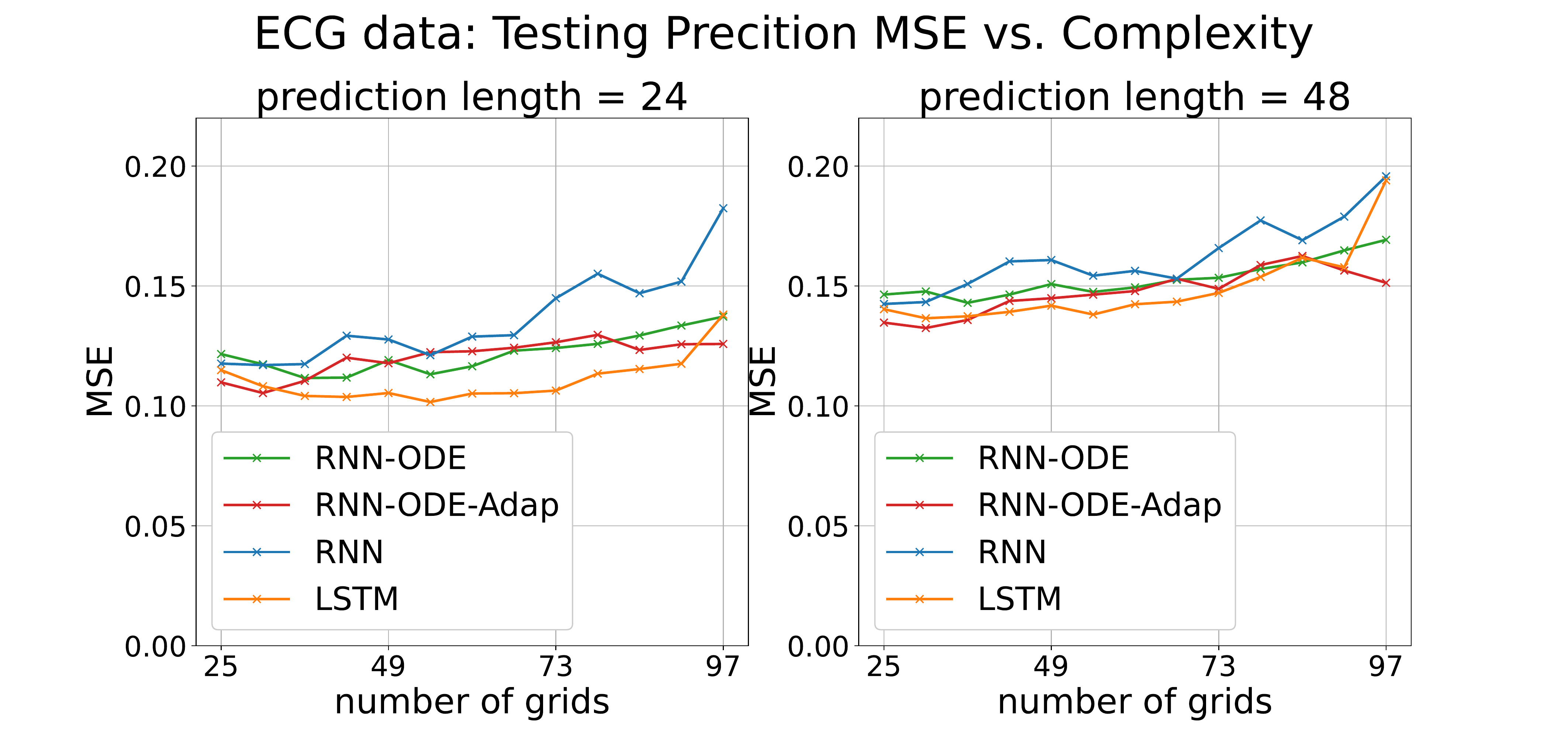}  & \hspace{-5em} \includegraphics[width=0.58\linewidth]{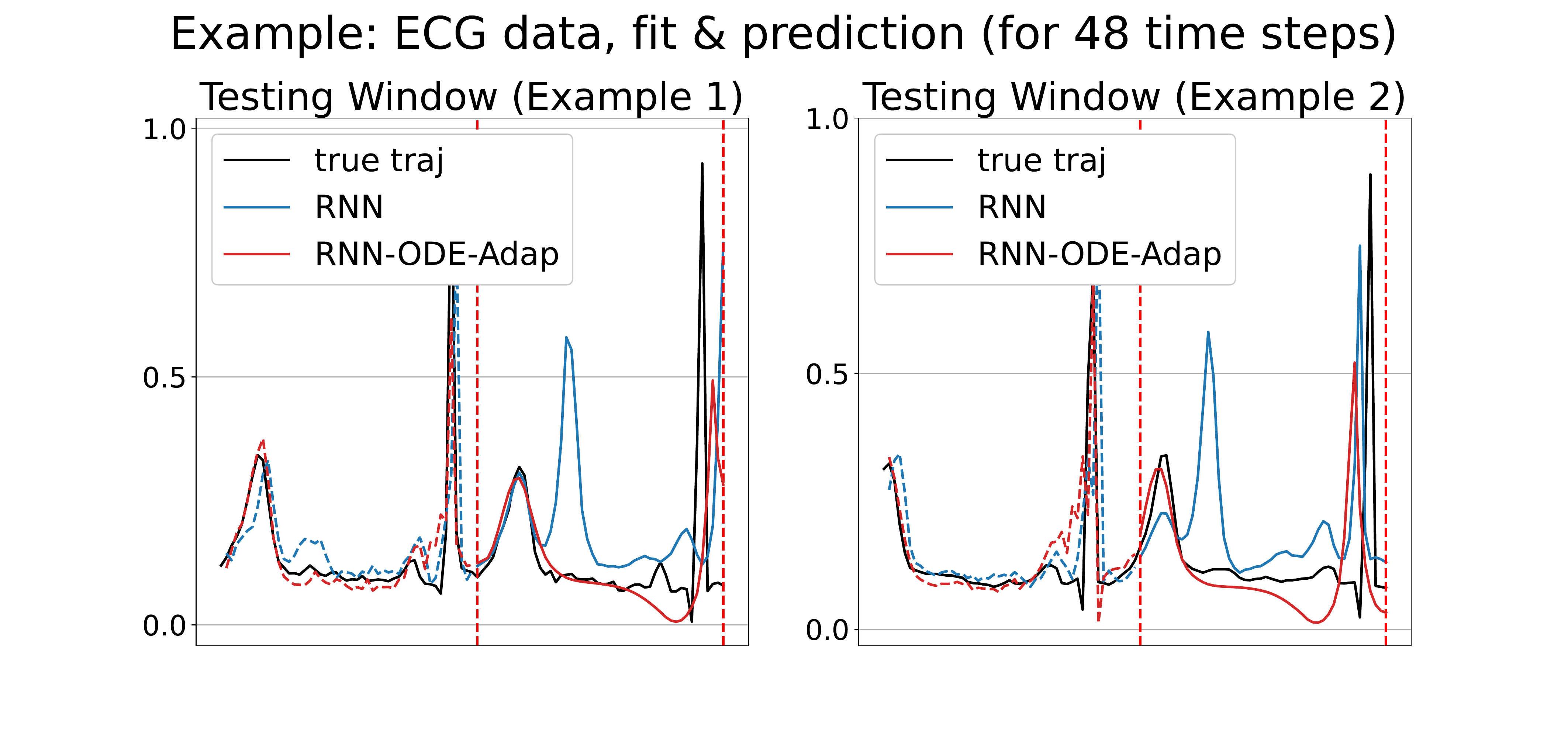} \\
    \end{tabular}
    \caption{Left: Comparison of the prediction errors on the real ECG data under two different prediction lengths (24 and 48) for RNN, LSTM, RNN-ODE, RNN-ODE-Adap. $x$-axis has been explained in the caption of Figure \ref{fig:spiral-error}. Right: Examples of 48 steps ahead prediction for the testing ECG data using RNN (marked in blue) and {\modelname} (marked in red). The predicted region is marked between dashed lines.}
    \label{fig:ECG-error}
\end{figure}

\section{Discussion}

In this paper we propose a general framework for constructing adaptive time steps when using the neural ODE combining the observed data to model times series. 
We demonstrate that it tends to be more efficient for modeling ``spike-like'' time series. The proposed algorithm of adaptive time steps is widely applicable to other types of models, not limited to neural ODE and RNN models. Moreover, the selection of adaptive time steps can be generalized to a broad class of non-stationary time series with different kinds of non-stationarities. 

This highlights the potential for further research in this filed, which may be approached from several angles. Firstly, there is a need for more theoretical analysis of the proposed {\modelname} framework under a broad spectrum of non-stationary time series, ranging from continuous to discontinuous data sequences. Additionally, there is a need to investigate the more flexible adaptive scheme, which could freely add middle steps adaptively and merge time grids without being restricted to dyadic partitioning.

% \xc{aknowledge}

\section*{Acknowledgement}

Y.T. and X.C. are partially supported by Simons Foundation (ID 814643) and NSF (DMS-2007040).

\bibliographystyle{plain}
\bibliography{refsArxiv}

\appendix

\setcounter{figure}{0}
\renewcommand{\thefigure}{A\arabic{figure}}

\renewcommand{\theequation}{A\arabic{equation}}
\setcounter{equation}{0}

\section{Proofs} \label{app:proofs}

\subsection{Proofs in Section \ref{subsec:theory-consistency} }\label{app:intensity-proof}

\begin{proof}[Proof of Proposition \ref{prop-error}]
    We consider the case with discretized and finite time grids. We assume there exists $d$ buffer time steps with samples $x_{-d},\ldots,x_{-1}$. The samples used for estimation are $x_{0},\ldots,x_{M}$, and each time step is with duration $\Delta t $. Thus the whole time duration is $T=M\Delta t$. The random process we observe on discrete time horizon $\{m:1\leq m\leq M\}$ is as follows. At time $m$ we observe integer variable $x_{m}\in\{0,1,2,\ldots\}$. Here $x_{m}$ means the number of event happening within $((m-1)\Delta t, m\Delta t]$ and $x_{m}=0$ means no event happening. Note that for the Hawkes process, which is essentially an inhomogeneous Poisson process, the variable $x_{m}$ is just a Poisson random variable with the intensity parameter depending on the historical observations. We denote the average intensity function within the time interval $((m-1)\Delta t, n\Delta t]$ as:
\[
\tilde\lambda(m) = \frac{1}{\Delta t}\int_{(m-1)\Delta t}^{m\Delta t} \lambda^*(t)dt,
\]
where $\lambda^*(t) = \mu + \alpha \int_0^t \phi(t-s)dN(s)$ is the true (continuous-time) intensity function. 

By the properties of the Poisson distribution, we have $\mathbb E[x_m|\mathcal F_{m-1}] = \tilde\lambda(m) \Delta t$ and $\mathrm{Var}[x_m|\mathcal F_{m-1}] = \tilde\lambda(m) \Delta t$. 
%For $\Delta t$ sufficiently small, we have $\mathbb P(\omega_n>1|\mathcal F_{n-1}) = 1-e^{-\tilde\lambda(n) \Delta t}-$
Our goal is to recover the intensity function $\lambda(\cdot)$ using the given observations. We consider the population loss function:
\[
\begin{aligned}
\Psi(\theta_h,\theta_d) & =\sum_{m=1}^M \mathbb E[ (x_m - F(m;\theta_h,\theta_d)\Delta t)^2 | x_{m-d},\ldots,\omega_{m-1} ]\\
& = \sum_{m=1}^M \Bigg\{ \mathbb E[ x_m^2 | x_{m-d}^{m-1}] - 2 \mathbb E[ x_m\cdot F(m;\theta_h,\theta_d)\Delta t | x_{m-d}^{m-1}]  + \mathbb E[  F^2(m;\theta_h,\theta_d)(\Delta t)^2 | x_{m-d}^{m-1}] \Bigg\}\\
& \propto \sum_{m=1}^M (\tilde \lambda(m) -  F(m;\theta_h,\theta_d))^2
\end{aligned}.
\]
Thus the optimizer will equal to $\tilde \lambda(n)$ as long as the function class $\textit{RNN-ODE}(d_{\mathrm{out}}, L_h,p_h, L_d,p_d)$ is rich enough to model the structure of the true intensity function.
\end{proof}

\begin{proof}[Proof of the claim in Remark \ref{rmk:theory}]
Note that when there is no event happening within the time interval $((m-1)\Delta t, m\Delta t]$ or when there is one event happening at $m\Delta t$, we have $|\tilde \lambda(m) - \lambda^*(t)|\leq C\alpha\Delta t$ where $C$ is a constant related to the Lipschitz constant of the influence kernel $\phi(\cdot)$. And when there is one event happening in $((m-1)\Delta t, m\Delta t)$, we have $|\tilde \lambda(n) - \lambda^*(t)|\leq \alpha$. By the concentration of Poisson distribution, there exists positive constant $M'$ such that there are at most $M'$ events happening within $[0,T]$ with high probability, and there is at most one event in each sub-interval $[(m-1)\Delta t, m\Delta t]$ for $M$ sufficiently large, we have that $\int_0^T |\tilde\lambda(t) - \lambda^*(t)|dt \leq C\alpha T \Delta t + M'\alpha\Delta t \to 0$ as $M\to\infty$.     
\end{proof}

\subsection{Proofs in Section \ref{subsec:approx-analysis}}\label{app:approx-proof}

% \vspace{0.5em}
\begin{remark}[Expressiveness of the model]
The hidden state $h(t)\in\R^{d_h}$ in \eqref{eq:dyn-sys} encodes the historical data, enabling $x(t)$ to be time-inhomogeneous. This raises the question regarding the expressiveness of Eq. \eqref{eq:dyn-sys} in representing a general dynamical system described by $x'(t)=F(x(t),t)$. 
There exist works that explored the expressiveness of the system $h'(t)=f_\theta(h(t),x(t)), y(t)=g_\phi(h(t))$, where $f_\theta, g_\phi$ are neural networks and $f_\theta$ possesses a RNN structure \cite{funahashi1993approximation, chow2000modeling, li2005approximation, li2022approximation}. Among these works, \cite{funahashi1993approximation, chow2000modeling, li2005approximation} assumed that $x(t)$ was generated from the underlying dynamics \eqref{eq:dyn-sys}, and thus the approximation problem was reduced to estimating $f$ and $g$ using neural networks $f_\theta$ and $g_\phi$. On the other hand, \cite{li2022approximation} took into account a broader range of input-output relationships. Specifically, it studied the expressiveness of the linear RNN structure in representing functionals $H_t$ that determined the output at time $t$ according to $H_t(\{x(\tau), \tau\in \mathcal{T}\})$, where $\mathcal{T}$ is an ordered index set (e.g., $\mathcal{T}=[0,T]$). \cite{li2022approximation} mainly focused on the case when $\{H_t(\{x(\tau), \tau\in \mathcal{T}\})\}$ is linear and time-homogeneous. 

Our approximation analysis bears more resemblance to the first category of studies and examines the approximation error for the discretely observed data.

\end{remark}
\vspace{0.5em}
\begin{remark}[Time-homogeneous dynamical systems]
For a time-homogeneous dynamical system $x'(t)=F(x(t))$, it can be represented as Eq. \eqref{eq:dyn-sys} by setting $d_h=D$, $f(h,x)=F(h)$, and $g(h)=h$. 
Theorems \ref{thm:continuous-error} and \ref{thm:discret-error} indicate that neural networks $f_\theta,g_\phi$ can be configured such that the observed data is approximated to any pre-specified accuracy. 
Prior studies \cite{funahashi1993approximation, chow2000modeling, li2005approximation, li2022approximation} proved that the system $x'(t)=F(x(t))$ could also be approximated using a continuous-time RNN, although without upper bounding the network size.

\end{remark}

\vspace{0.5em}

\paragraph{Proof of Lemma \ref{lemma:func-approx}}

Following the notations in \cite{{yarotsky2017error}}, we consider Sobolev space $\mathcal{W}^{n,\infty}([-1,1]^d)$, with $n=1,2,\ldots$, defined as the space of functions on $[-1,1]^d$ lying in $L^\infty$ with their weak derivatives up to order $n$.
From the proof of \cite[Theorem 1]{yarotsky2017error}, for any $f:[-1,1]^d\to\R$ such that $f\in \mathcal{W}^{n,\infty}([-1,1]^d)$ and $\epsilon>0$, there exists a neural network $\widetilde{f}$ such that $\max_{x\in [-1,1]^d}|f(x)-\widetilde{f}(x)| < \epsilon$, and $\widetilde{f}$ has $O(\ln(d+1)(\ln(\frac{\alpha_f}{\epsilon})+1))$ layers and $O(2^{d(d+1)}d^{d+2}\ln(d+1)^2 (\frac{2 \beta_f}{\epsilon})^{\frac{d}{n}}(\ln(\frac{\alpha_f}{\epsilon})+1))$ trainable parameters, where $\alpha_f=\|f\|_{\mathcal{W}^{n,\infty}([-1,1]^d)}\coloneqq \max_{\mathbf{n}:|\mathbf{n}|\leq n}\text{ess sup}_{x\in[-1,1]^d} |D^{\mathbf{n}}f(x)|, \beta_f \coloneqq \max_{\mathbf{n}:|\mathbf{n}|=1 }\text{ess sup}_{x\in[-1,1]^d} |D^{\mathbf{n}}f(x)|$.

In our case, we take $n=1$. For $f:[-1.1,1.1]^{d_h}\times [-1,1]^D\to \R^{d_h}$ that is Lipschitz in both $\eta$ and $x$, we define $L^{f,h}, L^{f,x}$ as follows:
\begin{equation}\label{eq:def-L^fh}
\begin{aligned}
L^{f,h} & \coloneqq \sup_{x\in [-1,1]^D}\sup_{\eta_1,\eta_2\in [-1.1,1.1]^{d_h}} \frac{\|f(\eta_1,x)-f(\eta_2,x)\|}{\|\eta_1-\eta_2\|},  \\
L^{f,x} & \coloneqq \sup_{h\in [-1.1,1.1]^{d_h}}\sup_{x_1,x_2\in [-1,1]^D} \frac{\|f(\eta,x_1)-f(\eta,x_2)\|}{\|x_1-x_2\|}.     
\end{aligned}
\end{equation}

For $\tilde{f}=(\tilde{f}_1,\dots,\tilde{f}_{d_h}):[-1,1]^{d_h}\times [-1,1]^D\to \R^{d_h}$ defined as $\tilde{f}(\tilde{\eta},x)\coloneqq f(1.1\tilde{\eta},x)$, we have that $\alpha_{\tilde{f}_i}\leq 1.1C_f, \beta_{\tilde{f}_i}\leq 1.1C_f$, $i=1,\dots,d_h$. Therefore, there exist $d_h$ subnetworks, denoted as $\hat{f}_1,\dots,\hat{f}_{d_h}$, such that 
\[\max_{\tilde{\eta}\in[-1,1]^{d_h},x\in[-1,1]^{D}}|\tilde{f}_i(\tilde{\eta},x)-\hat{f}_i(\tilde{\eta},x)|< \frac{\epsilon_f}{\sqrt{d_h}},\]
and each subnetwork has $O(\ln(\frac{1.1C_f}{\epsilon_f})+\ln d_h+1)$ layers 
% \lx{is it $C_f$ or $1.1C_f$?} 
and $O( (\frac{ 2.2C_f}{\epsilon_f})^{{d_h+D}}(\ln(\frac{1.1C_f}{\epsilon_f})+\ln d_h+1))$ weights, where the constants of big-$O$ notations depend on $d_h$ and $D$. 

Thus, we can construct $\tilde{f}_\theta$ as a network consisting of $d_h$ parallel sub-networks that implement each of $\hat{f}_i$. Then, for $f_\theta(\eta,x)\coloneqq \tilde{f}_\theta(\frac{1}{1.1}\eta,x)$,
\begin{align*}
\max_{\eta\in[-1.1,1.1]^{d_h},x\in[-1,1]^{D}}\|f(\eta,x)-f_\theta(\eta,x)\|_2 &= \max_{\tilde{\eta}\in[-1,1]^{d_h},x\in[-1,1]^{D}}\|\tilde{f}(\tilde{\eta},x)-\tilde{f}_\theta(\tilde{\eta},x)\|_2 \\
&\leq \sqrt{d_h}\max_{\tilde{\eta}\in[-1,1]^{d_h},x\in[-1,1]^{D}} \|\tilde{f}(\tilde{\eta},x)-\tilde{f}_\theta(\tilde{\eta},x)\|_\infty<\epsilon_f.    
\end{align*}
$f_\theta$ has $O(\ln(\frac{1.1C_f}{\epsilon_f})+\ln d_h+1)$ layers and $O( (\frac{ 2.2C_f}{\epsilon})^{{d_h+D}}(\ln(\frac{1.1C_f}{\epsilon_f})+\ln d_h+1))$ weights, where the constants of big-$O$ notations depend on $d_h$ and $D$. Specifically, 
\begin{align*}
&\# (\text{layers of } f_\theta) \leq C\ln(d_h+D+1)(\ln(\frac{1.1C_f}{\epsilon_f})+\ln d_h+1),\\
&\# (\text{weights of } f_\theta) \leq C2^{(d_h+D)(d_h+D+1)}(d_h+D)^{d_h+D+2}d_h^{\frac{d_h+D+2}{2}}\ln(d_h+D+1)^2 \\
&\qquad\qquad\qquad\qquad\qquad\cdot(\frac{ 2.2C_f}{\epsilon})^{{d_h+D}}(\ln(\frac{1.1C_f}{\epsilon_f})+\ln d_h+1),    
\end{align*}
for some absolute constant $C>0$. 
$g_\phi$ can be constructed similarly. We define $\tilde g=(\tilde{g}_1,\dots,\tilde{g}_{D'}): [-1,1]^{d_h}\to \R^{D'}$ as $\tilde g(\tilde{\eta})\coloneqq g(1.1\tilde{\eta})$. Then, there exist $D'$ subnetworks, denotes as $\hat{g}_1,\dots,\hat{g}_{D'}$, such that
\[\max_{\tilde{\eta}\in [-1,1]^{d_h}} \left\|\tilde{g}_i(\tilde{\eta}) - \hat{g}_i(\tilde{\eta})\right\|_2 < \frac{\epsilon_g}{\sqrt{D'}},\quad i=1,\dots,D',\]
and each subnetwork has $O(\ln(\frac{1.1C_g}{\epsilon_g})+\ln D'+1)$ layers and $O( (\frac{ 2.2C_g}{\epsilon_g})^{d_h}(\ln(\frac{1.1C_g}{\epsilon_g})+\ln D'+1))$ weights, where the constants of big-$O$ notations depend on $d_h$ and $D'$. We construct $\tilde{g}_\phi$ as a network consisting of $D'$ parallel subnetworks that implements $\{\hat{{g}_i}\}$. Then, for $g_\phi(\eta)\coloneqq \tilde{g}_\phi(\frac{1}{1.1}\eta)$,
\[\max_{\eta\in [-1.1,1.1]^{d_h}} \left\|g(\eta) - g_\phi(\eta)\right\|_2 < \epsilon_g,\]
and
\begin{align*}
&\# (\text{layers of } g_\phi) \leq C\ln(d_h+1)(\ln(\frac{1.1C_g}{\epsilon_g})+\ln D'+1),\\
&\# (\text{weights of } g_\phi) \leq C2^{d_h(d_h+1)}d_h^{d_h+2}D'^{\frac{d_h+2}{2}}\ln(d_h+1)^2 (\frac{2.2C_g}{\epsilon_g})^{d_h} \cdot(\ln(\frac{1.1C_g}{\epsilon_g})+\ln D'+1).
%\\&\qquad\qquad\qquad\qquad\qquad
\end{align*}
This proves the claim.

\paragraph{Proof of Theorem \ref{thm:continuous-error}. }  

\begin{proof}[Proof of Theorem \ref{thm:continuous-error}]

We denote $u(t)\coloneqq \|h(t) - h_{\text{NN}}(t)\|$, and $t_0 = \inf_{t\in [0,T]}\{u(t) \geq 0.1\}$.  Since $u(0)=0$ and $u(t)$ is continuous, we know that $t_0>0$. In the following, we show that $t_0=T$ by contradiction. Otherwise, suppose that $t_0 < T$. Then for $t\in [0,t_0]$, $u(t)=\|h(t) - h_{\text{NN}}(t)\|\leq 0.1$, which implies that $h_{\text{NN}}(t) \in [-1.1,1.1]^{d_h}$.

Then, by \eqref{eq:func-approx}, for $t\in [0,t_0]$,
\begin{align*}
u(t) &= \left\|\int_0^{t} \left(f(h(s),x(s)) - f_\theta(h_{\text{NN}}(s), x(s))\right)\mathrm{d}s\right\| \\
&\leq   \int_0^{t}\left\|f(h(s),x(s)) - f_\theta(h_{\text{NN}}(s), x(s))\right\| \mathrm{d}s \\
&\leq \int_0^{t} \left(\epsilon_f + L(s)\|h(s) - h_{\text{NN}}(s)\|\right) \mathrm{d}s, 
\end{align*}
where 
\[L(s) = L^{f,h}_i,\quad \text{if } s\in [t_{i-1},t_i], i=1,\dots,n+1, \]
and $\{t_i\}_{i=1}^n$ the time grid corresponding to the partition $\mathcal{D}_1\cup\mathcal{D}_2$ such that $\frac{1}{T}\sum_{i=1}^{n+1} L^{f,h}_i (t_i-t_{i-1})\leq\frac{1}{T} ( L_{\rm low} |D_1| + L_{\rm high} |D_2|)= L^{(\rm avg)}$, and $L^{f,h}_i$ is defined as in \eqref{eq:def-local-Lipschitz} with taking the interval $[s,t]=[t_{i-1},t_i]$.
Therefore, 
\[u(t)\leq \epsilon_ft + \int_0^t L(s)u(s)\mathrm{d}s,\quad t\in [0,t_0].\]
By the Gr$\ddot{\text{o}}$nwall's inequality,
\[u(t) \leq \epsilon_ft\exp(\int_0^t L(s)\mathrm{d}s) \leq \epsilon_fT\exp(\sum_{i=1}^{n+1} L^{f,h}_i (t_i-t_{i-1})) \leq \epsilon_fT\exp(L^{(\rm avg)}T) < 0.1, \quad t\in [0,t_0].\]
Specifically,
\[u(t_0) \leq \epsilon_fT\exp(L^{(\rm avg)}T) < 0.1.\]
Since $u(t)$ is continuous, there exists a sufficiently small $\delta>0$, such that $u(t) <  0.1$ for $t\in [t_0,t_0+\delta]$. This is a contradiction to the definition of $t_0$. Thus, we conclude that $t_0 = T$, and therefore $h_{\text{NN}}(t)\in [-1.1,1.1]^{d_h}$, $\forall t\in [0,T]$. By the similar analysis above, we have that
\[u(t) \leq \epsilon_fT\exp(L^{(\rm avg)}T),\quad t\in [0,T].\]
Thus, for $t\in [0,T]$,
\begin{align*}
\|y(t)-y_{\text{NN}}(t)\| &\leq \| g(h(t))-g(h_{\text{NN}}(t))\| + \| g(h_{\text{NN}}(t))-g_\phi(h_{\text{NN}}(t))\| \\
&\leq L_gT\exp(L^{(\rm avg)}T)\epsilon_f + \epsilon_g,
\end{align*}
which proves the claim.
\end{proof}

\paragraph{Proof of Theorem \ref{thm:discret-error}. }

\begin{proof}[Proof of Theorem \ref{thm:discret-error}]

Denote $\varepsilon_i = h(t_i) - \hat{h}_{\text{NN}}(t_i)$, then $\varepsilon_0=0$.
In the following, we apply the induction argument, iteratively showing that
\begin{equation}\label{eq:proof1}
\|\varepsilon_i\|\leq e_h<0.1,\quad \hat{h}_{\text{NN}}(t_i)\in [-1.1,1.1]^{d_h},\quad i=0,\dots, N,    
\end{equation}
where
\[e_h\coloneqq T\exp(\sum_{i=1}^N L^{f,h}_i\Delta t_i)  \left(\epsilon_f + \max_{j}\{\mu_j\Delta t_j\}\right).\]

For $i=0$, \eqref{eq:proof1} naturally holds since $h(0)=\hat{h}_{\text{NN}}(0)$.
If \eqref{eq:proof1} holds For $i\leq k$, then for $i=k+1$, by $\|\varepsilon_k\| < 0.1$,  $\hat{h}_{\text{NN}}(t_{k})\in [-1.1,1.1]^{d_h}$. From the definition of $\varepsilon_{k+1}$, 
\begin{align*}
\varepsilon_{k+1} &= \left\|\left(h(t_{k}) + \int_{t_{k}}^{t_{k+1}} f(h(s),x(s))\mathrm{d}s\right) - \left(\hat{h}_{\text{NN}}(t_{k}) + \Delta t_{k+1} f_\theta(\hat{h}_{\text{NN}}(t_{k}), x(t_{k}))\right)\right\|   \\
&\leq \varepsilon_{k} + \int_{t_{k}}^{t_{k+1}} \left\|f(h(s),x(s))-f_\theta(\hat{h}_{\text{NN}}(t_{k}), x(t_{k}))\right\|\mathrm{d}s.
\end{align*}
Next, we upper bound the second term. By the triangle inequality, $\hat{h}_{\text{NN}}(t_{k})\in [-1.1,1.1]^{d_h}$ and \eqref{eq:func-approx}, for $s\in [t_{k},t_{k+1}]$,
\begin{align*}
&\quad\left\|f(h(s),x(s))-f_\theta(\hat{h}_{\text{NN}}(t_{k}), x(t_{k}))\right\| \\
&\leq  \left\|f(h(s),x(s))-f(\hat{h}_{\text{NN}}(t_{k}),x(t_{k})\right\| + \left\|f(\hat{h}_{\text{NN}}(t_{k}),x(t_{k})-f_\theta(\hat{h}_{\text{NN}}(t_{k}),x(t_{k}))\right\|\\
&\leq \left\|f(h(s),x(s))-f(h(t_{k}),x(s))\right\| + \left\|f(h(t_{k}),x(s))-f(h(t_{k}),x(t_{k}))\right\| \\
&\quad + \left\|f(h(t_{k}),x(t_{k})) - f(\hat{h}_{\text{NN}}(t_{k}),x(t_{k}))\right\| + \epsilon_f \\
&\leq    (L^{f,h}_{k+1}M_{k+1}^f +L^{f,x}_{k+1}L^x_{k+1}  )\Delta t_{k+1} + L^{f,h}_{k+1}\Delta t_{k+1}\varepsilon_{k} +  \epsilon_f,
\end{align*}
where the first component is due to $\|h(s)-h(t_{k})\|=\|\int_{t_{k}}^s f(h(u),x(u))du\|\leq M_{k+1}^f \Delta t_{k+1}$ and the second term results from $|x(s)-x(t_k)|\leq L^x_{k+1}\Delta t_{k+1}$.

This implies that 
\begin{equation}\label{eq1}
\varepsilon_{k+1} \leq (1+L^{f,h}_{k+1}\Delta t_{k+1})\varepsilon_{k} + \gamma_{k+1},    
\end{equation}
where 
\[\gamma_{k+1}\coloneqq \epsilon_f\Delta t_{k+1} + (L^{f,h}_{k+1}M_{k+1}^f +L^{f,x}_{k+1}L^x_{k+1})\Delta t_{k+1}^2 = \epsilon_f\Delta t_{k+1} +\mu_{k+1}\Delta t_{k+1}^2.\]

From \eqref{eq1}, we obtain that
\begin{align*}
\varepsilon_{k+1} \leq \sum_{j=1}^{i+1} \left(\gamma_j \cdot \prod_{l=j+1}^{k+1} (1+L^{f,h}_l\Delta t_l)\right).    
\end{align*}
Since $1+x\leq \exp(x)$, 
\[\prod_{l=j+1}^{k+1} (1+L^{f,h}_l\Delta t_l)\leq \exp(\sum_{l=j+1}^{k+1}L^{f,h}_l\Delta t_l)\leq \exp(\sum_{l=1}^{N}L^{f,h}_l\Delta t_l).\] 
Hence,  we have
\begin{align*}
\varepsilon_{k+1} &\leq \exp(\sum_{i=1}^{N}L^{f,h}_i\Delta t_i)\sum_{j=1}^{N} \gamma_j = \exp(L^{f,h}T)\left( \sum_{j=1}^N\epsilon_f\Delta t_j + \mu_j\Delta t_j^2\right) \\
&\leq T\exp(\sum_{i=1}^N L^{f,h}_i\Delta t_i)  \left(\epsilon_f + \max_{j}\{\mu_j\Delta t_j\}\right) = e_h.
\end{align*}
Therefore, \eqref{eq:proof1} holds for $i=k+1$. By the induction argument, \eqref{eq:proof1} is true for $i=1,\dots,N$.

Finally, for $i=1,\dots,N$, by triangle inequality and the fact that $\hat{h}_{\text{NN}}(t_i)\in [-1.1,1.1]^{d_h}$, applying \eqref{eq:func-approx} results in
\begin{align}
\|y(t_i)-\hat{y}_{\text{NN}}(t_i)\|
&=\| g(h(t_i)) - g_\phi(\hat{h}_{\text{NN}}(t_i))\|
\nonumber\\
&\leq \| g(h(t_i))-g(\hat{h}_{\text{NN}}(t_i))\| + \| g(\hat{h}_{\text{NN}}(t_i))-g_\phi(\hat{h}_{\text{NN}}(t_i))\| \\
& \leq \epsilon_g + L_g\varepsilon_{k+1} < \epsilon_g + L_ge_h.   \nonumber 
\end{align}
This proves the claims in Theorem \ref{thm:discret-error}.
\end{proof}

\section{Experimental Details}\label{sec:append-exp}

\subsection{Implementation Details}\label{sec:detail}

Given the training windows $\{\boldsymbol{x}^{(\Tr,k)}\}_{k=1}^{K^{(\Tr)}}$, the algorithm \ref{alg_adaptive} is used as a preprocessing step to prepare {\it each} training window to the irregular sub-window with adaptive time steps. The resulting adaptive training windows are then used to train the neural ODE model \eqref{eq:ode}-\eqref{eq:output} using the mean-squared loss function \eqref{eq:training_loss1}.
During the inference phase, the learned ODE model \eqref{eq:ode} will be used for fitting and prediction tasks. It can be used for arbitrary and irregular (future) time steps. 

\paragraph{Choice of Monitor Functions.}
    The monitor function in Algorithm \ref{alg_adaptive} can be chosen flexibly, not restricted to the maximum variation defined in \eqref{eq:variation}. Since this work mainly uses non-stationary time series with ``spike''-like patterns as an example, the monitor function \eqref{eq:variation} is a natural choice for identifying abrupt ``spikes''. In general,  the monitor function may be designed case-by-case depending on the problem context. For example, when modeling the event-type counting process (such as the Hawkes process simulated in Section \ref{sec:hawkes}), where $x_i\in \mathbb{N}$ is the number of events in the current time interval, we may choose to use    
    the maximum counts $M(x_i,\ldots,x_j) :=\max\{x_{i}, \ldots, x_{j}\}$. By setting the threshold $\epsilon \in (0,1)$, such a monitor function will assign the finest time steps to intervals with events ($x_i>0$) and use rough time steps for regions without events ($x_i=0$).   

\paragraph{Choice of Threshold in Algorithm \ref{alg_adaptive}.} The choice of the selection threshold $\epsilon$ used in Algorithm \ref{alg_adaptive} can be selected from training data via simulation. In detail, note that a larger threshold $\epsilon$ would lead to a sparser set of selected time steps (the output of Algorithm \ref{alg_adaptive}). Therefore, we primarily determine the threshold $\epsilon$ by calibrating the number of remaining time stamps after applying Algorithm \ref{alg_adaptive}, allowing us to control the desired efficiency. 
{Specifically, we employ a validation data set to calibrate the selection of threshold values, ensuring that the chosen $\epsilon$ yields the desired average lengths for the adaptively selected time steps. As illustrated in Figure \ref{fig:ECG_thres}, an example of ECG data demonstrates the influence of the threshold parameter $\epsilon$ on the chosen grids. It can be observed that an increasing $\epsilon$ leads to a reduction in the number of selected grids. Furthermore, for each value of $\epsilon$, the chosen grids correspond to sub-intervals with greater variation.
}

\begin{figure}[t]
    \centering
\includegraphics[width=1\linewidth]{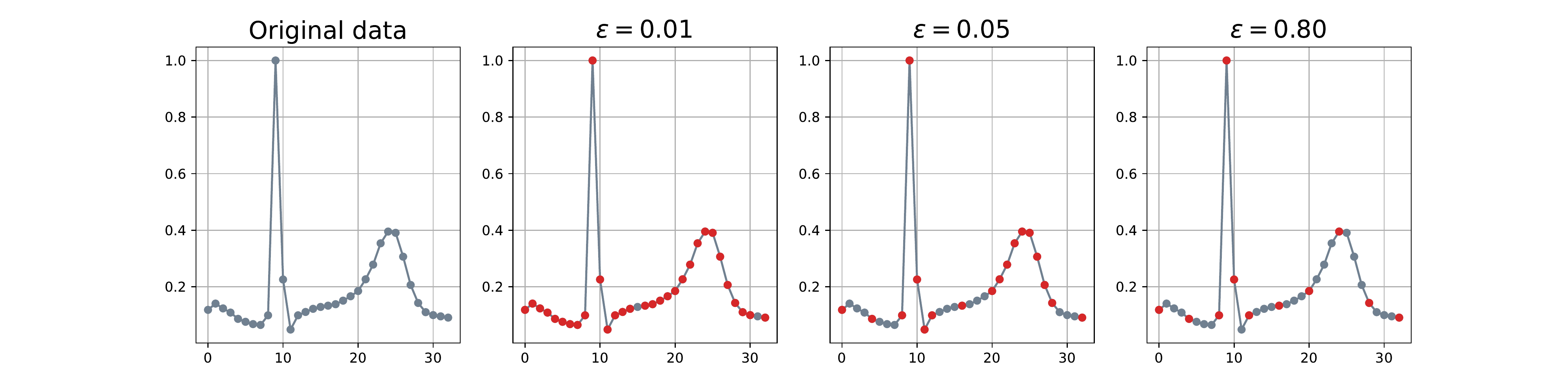}
    \caption{Illustration of adaptively selected time steps with different thresholds (ECG data). The gray dots depict the original data points and the red points illustrate the time steps selected by Algorithm \ref{alg_adaptive} using threshold $\epsilon=0.01$, $0.05$, and $0.8$, and $L=3$.}
    \label{fig:ECG_thres}
\end{figure}

\paragraph{Evaluation Metric.}

To compare the {\it multi-step prediction} performance of different time series models, we use the mean-squared multi-step ahead prediction error as follows. After obtaining the fitted model, we can use the trained networks \eqref{eq:ode}-\eqref{eq:output} to make predictions on the testing windows $\{\boldsymbol{x}^{(\Te,k)}\}_{k=1}^{K^{(\Te)}}$, where $\boldsymbol{x}^{(\Te,k)}=\{x^{(\Te, k)}(t^{(\Te,k)}_{1}),\dots,x^{(\Te, k)}(t^{(\Te,k)}_{n})\}$.
Given a historical trajectory $\{x(t_{1}),\dots,x(t_{n})\}$, we can apply the fitted model to perform multi-step ahead prediction
\begin{align}
    \hat h(t_{i+1}) &=\hat h(t_{i}) + \begin{cases}
     \int_{t_{i}}^{t_{i+1}} f(\hat h(s), x(s); \theta_h)ds, \quad & \text{when } i\leq n,\\
     \int_{t_{i}}^{t_{i+1}} f(\hat h(s), \hat x(s); \theta_h)ds, \quad & \text{when } n<i\leq n+m,\end{cases}\label{eq:predict}\\
    \hat x(t_{i+1}) &= g(\hat h(t_{i+1});\theta_d),\notag
\end{align}
which will be iteratively solved for $i=1,\ldots,n+m$. 
The first ODE can be solved by, for instance, the Euler method.  
When comparing the $m$-step ahead prediction performance of different methods, we use the averaged $\ell_2$ norm of the prediction error of length $m$; specifically, we take $n=\lfloor N/2\rfloor, m=N-n$ in the experiments in Section \ref{sec:numerical} and perform the prediction in Eq. \eqref{eq:predict} for each testing window $\boldsymbol{x}^{(\Te,k)}$ with the predicted value denoted as $\hat{x}^{(\Te, k)}(\cdot)$, then the resulting prediction performance on test data is measured as follows
\begin{align}\label{eq:pred-error}
\text{MSE}_{\text{pred}} &= \frac{1}{K^{(\Te)}}\sum_{k=1}^{K^{(\Te)}}\left(\frac{1}{m}\sum_{i=n+1}^{n+m} \left\|x^{(\Te, k)}(t_{i}^{(\Te,k)}) - \hat{x}^{(\Te, k)}(t_{i}^{(\Te,k)})\right\|^2\right)^{1/2}.%\label{eq:MSE}
\end{align} 
We could also use other reasonable metrics that measure the discrepancy between times series data, such as the averaged $\ell_1$ norm or the dynamic time warping distance \cite{sakoe1978dynamic}.

To compare the {\it one-step prediction} performance of different methods on the intensity function of the event data generated from Hawkes processes which is described in Section \ref{sec:hawkes}, we employ the error as defined in \eqref{eq:fit-error}, which has the similar form to the training loss \eqref{eq:training_loss1}: 
\begin{align}\label{eq:fit-error}
\text{MSE}_{\text{fit}} = \frac{1}{K^{(\Te)}}\sum_{k=1}^{K^{(\Te)}}\sum_{i=1}^{N}  \left\| \hat{\lambda}^{(\Te, k)}(t_{i}^{(\Te,k)}) - \lambda^{(\Te, k)}(t_{i}^{(\Te,k)}) \right\|^2 |t_{i}^{(\Te,k)} - t_{i-1}^{(\Te,k)}|,
\end{align}
here the superscript $^{(\Te)}$ denotes that the error is evaluated on the testing data. $\lambda$ and $\hat{\lambda}$ denote the true and fitted intensity functions of the event data, respectively. In this case, $\hat{\lambda}$ is obtained by iteratively solving \eqref{eq:predict} with $n=N-1$ and $m=0$.

\paragraph{Buffer Steps.}

To facilitate training and improve the performance of the models, we leverage additional ``buffer steps'' at the beginning of each window to mitigate the effect of the zero initialization of the hidden states. Buffer steps refer to the additionally padded time stamps before each window. Specifically, for the $k$-th training window $\{x^{(\Tr,k)}(t_{i}^{(\Tr,k)})\}_{i=1}^{N}$, adding $m$ buffer steps means that the original time series is augmented to $\{x^{(\Tr,k)}(t_{i}^{(\Tr,k)})\}_{i=-m}^{N}$, where for $i=-m,\dots,-1$, $t_{i+1}^{(\Tr,k)}-t_{i}^{(\Tr,k)} =\Delta t \coloneqq \min_{i=0,\dots,N-1}\{t_{i+1}^{(\Tr,k)}-t_{i}^{(\Tr,k)}\}$.
{An illustration of the buffer steps for the discrete point process data is shown in Figure \ref{fig:Hawkes_buffer_ex}.}
Detailed information on buffer steps, pertaining to the experiments in Section \ref{sec:numerical}, along with additional experiments examining the impact of incorporating buffer steps and selecting the appropriate number of buffer steps using validation data, can be found in Appendix \ref{sec:append-buffer}.

\begin{figure}[tp]
    \centering
\includegraphics[width=0.45\textwidth]{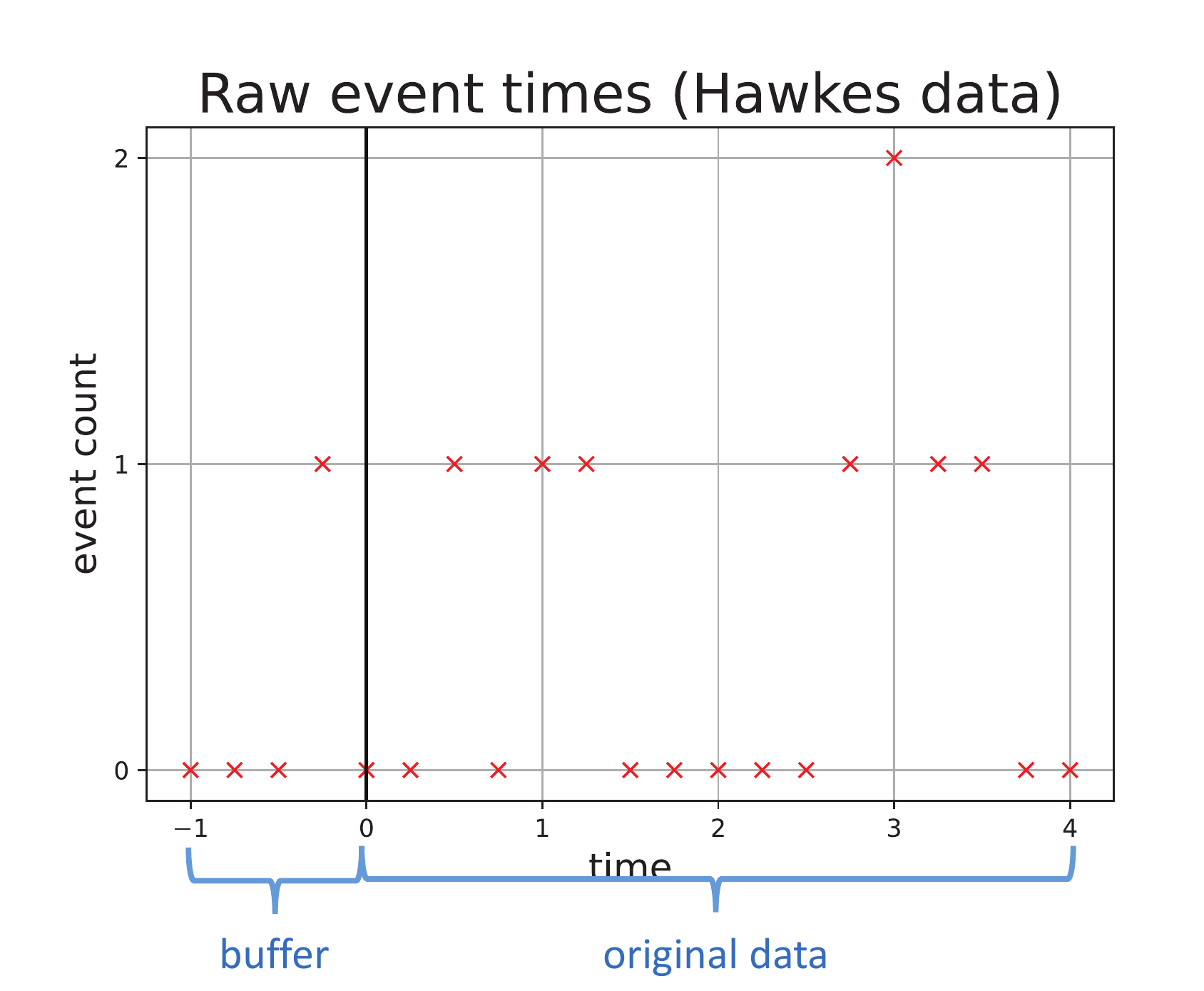}
    \caption{Illustration of buffer steps 
    constructed on a discrete-time event data generated from a Hawkes process. }
    \label{fig:Hawkes_buffer_ex}
\end{figure}

\subsection{Network Structure}\label{sec:net-struct}

In the experiments, we use the same network structure for RNN, RNN-ODE, and RNN-ODE-Adap, namely the ODE function $f$ follows the vanilla RNN structure
\[f(h, x;\Theta_f) = \text{tanh}(W_f[h,x]+b_f), \]
where $h\in\R^{d_{h}}, W_f\in\R^{d_{h}\times (d_{h}+D)}, b_f\in\R^{d_{h}}$, and $\Theta_f=\{W_f,b_f\}$. RNN updates the hidden states discretely by $h_{t_n} = f(h_{t_{n-1}}, x_{t_n};\Theta_f)$.

Furthermore, in this paper, the output function $g$ for RNN, RNN-ODE, and RNN-ODE-Adap is taken as a fully connected (FC) layer
\begin{align}\label{eq:outputfunc}
 g(h;\Theta_g) = W_gh+b_g,   
\end{align}
where $h\in\R^{d_{h}}, W_g\in\R^{D\times d_{h}}, b_g\in\R^D$, and $\Theta_g=\{W_g, b_g\}$. In all the experiments, we take $d_{h}=128$. 

For the LSTM model, we use the same output function as in \eqref{eq:outputfunc} to decode hidden states and the vanilla LSTM block to update hidden states. The latter is detailed as
\begin{align*}
f_{\text{LSTM}}(h,c,x; \Theta_{f_{\text{LSTM}}}) &= o(h,x; \Theta_o) \odot \text{tanh}\left(\mathtt{c}(h,c,x; \Theta_{\mathtt{c}})\right),  % \\
\end{align*}
where
\begin{align*}
 o(h,x; \Theta_o) &= \sigma(W_o[h,x]+b_o),\\
 \mathtt{c}(h,c,x; \Theta_{\mathtt{c}}) &= p(h,x; \Theta_p)\odot c + i(h,x; \Theta_i) \odot q(h,x; \Theta_q),\\
i(h,x; \Theta_i) &= \sigma(W_i[h,x]+b_i),\\
p(h,x; \Theta_p) &= \sigma(W_p[h,x]+b_p),\\
q(h,x; \Theta_q) &= \sigma(W_q[h,x]+b_q),
\end{align*}
here $h,c\in \R^{d_{h}}, W_o, W_i, W_p, W_q\in \R^{d_{h}\times (d_{h}+D)}, b_o, b_i, b_p, b_q\in \R^{d_{h}}$, and the parameters in the LSTM model is denoted as $\Theta_{f_{\text{LSTM}}} = \{W_o, W_i, W_p, W_q, b_o, b_i, b_p, b_q\}$. LSTM updates the cell states and hidden states iteratively by $c_{t_n}=\mathtt{c}(h_{t_{n-1}}, c_{t_{n-1}}, x_{t_n};\Theta_{\mathtt{c}}), h_{t_n} = f_{\text{LSTM}}(h_{t_{n-1}}, c_{t_{n-1}}, x_{t_n};\Theta_f) = o(h_{t_{n-1}}, x_{t_n};\Theta_o)\odot \text{tanh}(c_{t_n})$.

\subsection{Training, Validation, and Testing Data Sets}

\subsubsection{Windows of the Finest Grids}\label{sec:append-window-finest}

In all the experiments, the results are obtained from multiple replicas. In each replica, training, validation, and testing windows of the finest grids are independently generated, and then used for training the neural networks, validating, and evaluating performance.

For the data in the spiral example, 50 of the total 500 training windows of length 65 are randomly chosen as validation data, and there are 500 testing windows of the same length. Each spiral follows the ODE system described in Section \ref{sec:spiral},  with $A$ perturbed. 5 windows are randomly chosen for each spiral sampled at 200 regular time steps. 

For event-time data generated from Hawkes process, 200 of the total 2000 training Hawkes sequences are randomly chosen as validation data, and there are 1000 testing sequences. Each sequence is generated with $\alpha=0.5, \mu(t)\equiv 0.5$ and an exponential kernel $\varphi(t) = 2 e^{-2 t}$. The data lies in physical time $[1,5]$. 
Note that the original data set only consists of the time stamps when the events happen. We need to further preprocess the original data to time series that indicate the number of events happening in small intervals. Specifically, we discretize the time space into uniform bins, transforming the continuous-time event times into discrete counts.

For ECG data, we select ten patients from the PTB-XL ECG dataset. For each patient, we have the 12-lead ECGs of 10-second length, with 50Hz frequency. We use 3-lead in our training and testing. Thus, there are in total 30 trajectories of length 500. The first 70\% and last 30\% of each trajectory are used to extract training and testing windows respectively to avoid overlapping. 300 of the total 3000 training windows are randomly chosen as validation data, and there are 900 testing windows. In this case, for each of the 30 trajectories, 100 training windows of length 97 are taken from the first 350 time steps, and 30 testing windows of the same length are from the last 150 time steps.

\subsubsection{Windows of the Predetermined Lengths}\label{sec:append-window-predeterm}

To get windows of a certain length, for RNN, LSTM, and RNN-ODE that use regular time steps, the original windows are interpolated to get regular time steps with the desired number of grids; RNN-ODE-Adap selects the time steps by adjusting hyper-parameters $\epsilon$ and $L$ in Algorithm \ref{alg_adaptive}, such that the averaged length of the adaptively selected validation windows is close to the desired length. 

The situation is different for the event-type data since the time when the event happens is available. In this case, RNN, LSTM, and RNN-ODE utilize the windows that count the number of events happening in time intervals formed by regular time steps of the required length. RNN-ODE-Adap first generates longer windows (and smaller $\Delta_t$) and then takes $L=1, \epsilon=0.5$ in Algorithm \ref{alg_adaptive} to generate windows with similar lengths to the required one. In this way, RNN-ODE-Adap utilizes windows with irregular time steps.

\subsection{More Details on Buffer Steps}\label{sec:append-buffer}

For RNN, LSTM, and RNN-ODE that use regular training time series $\{x_{t_i}\}_{i=0}^n$ with $t_{i+1}-t_i = \Delta t\; (i=0,\dots,n-1)$, adding $m$ buffer steps means that the original time series is augmented to $\{x_{t_i}\}_{i=-m}^n$, with $t_{i+1}-t_i = \Delta t\; (i=-m,\dots,n-1)$. 
For spiral and ECG data, we take $m=2$ and $x_{-2}=x_{-1}=x_0$. For the event-type data from the Hawkes process, we take $m=\lfloor \frac{n}{4} \rfloor$ and $\{x_{t_i}\}_{i=-m}^{-1}$ as the true event data, in this way $t_0 - t_{-m} = m \Delta t =  \lfloor \frac{n}{4} \rfloor \frac{4}{n} \approx 1$. Here $\Delta_t = \frac{4}{n}$ due to that the event-time data are generated in an interval with physical time $4$.

For RNN-ODE-Adap that is trained with irregular training time series $\{x_{t_i}\}_{i=0}^n$, we first find the minimal increment in time $\Delta t \coloneqq \min_i\{t_{i+1}-t_i\}$, then the series with $m$ buffer steps added is $\{x_{t_i}\}_{i=-m}^n$, with $t_{i+1}-t_i = \Delta t\; (i=-m,\dots,-1)$. For spiral and ECG data, we still take $m=2$ and $x_{-2}=x_{-1}=x_0$. For the event-time data, we use the same number of buffer steps as the other three methods for a fair comparison.

{Figure \ref{fig:Hawkes-buffer2} below presents two examples of the event-type data from the Hawkes process, illustrating the performance improvement achieved by incorporating buffer steps, which mitigate the effects of zero-initialized hidden states. The light and dark green lines represent the fitted intensity of RNN-ODE without and with buffer steps, respectively. It can be observed that, in the absence of buffer steps, the initial few steps are not accurately estimated due to zero initialization. The inclusion of buffer steps effectively eliminates this issue.}

\begin{figure}[htpb]
\centering
\includegraphics[width=\textwidth]{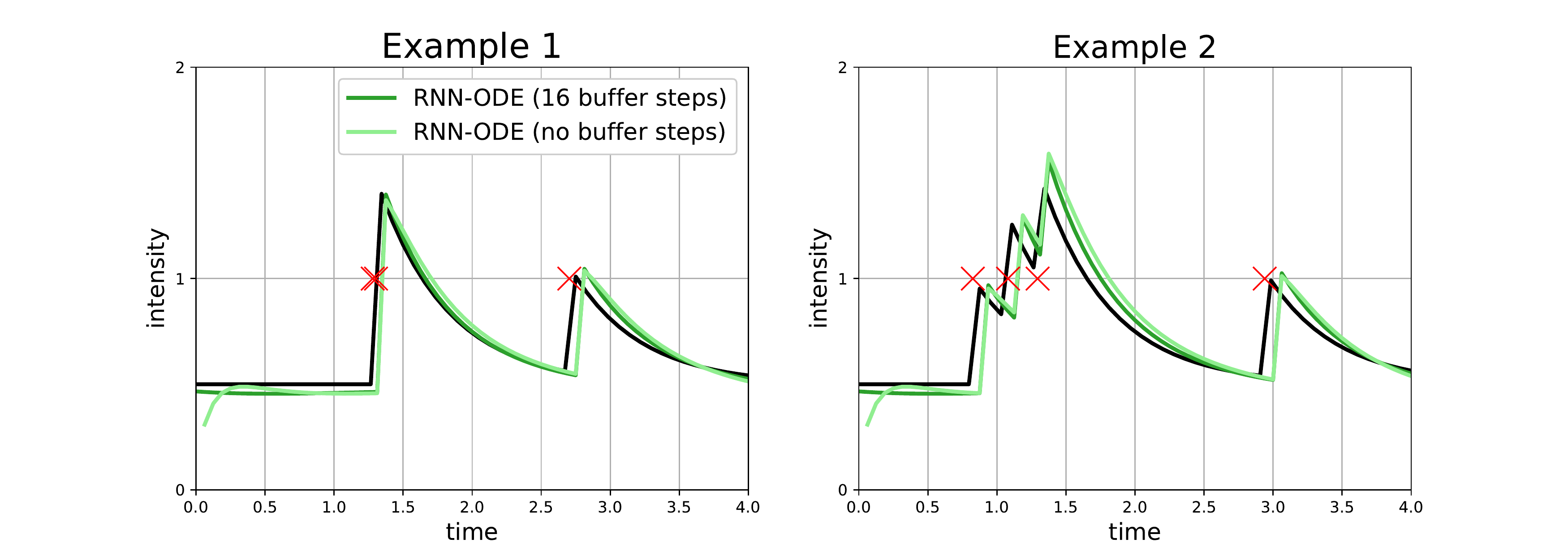}
\caption{Comparison of RNN-ODE models with or without buffer steps on the two examples of the discrete event-time data generated from the Hawkes process (the number of grids is 65).}
\label{fig:Hawkes-buffer2}
\end{figure}

We investigate more on the number of buffer steps for the event-type data. Specifically, Figure \ref{fig:Hawkes-buffer} below shows the fitting errors of RNN-ODE for different buffer steps when the number of grids is 65.  We remark that $\Delta t$ keeps the same for all the number of buffer steps, thus the number of buffer steps also reflects the physical buffer time used. It can be observed that as the number of buffer steps increases from 2 to 16, the fitting error decreases. This implies that for the data with long history dependence like the Hawkes process, enough buffer steps should be kept to circumvent non-stationary results. 
\begin{figure}[h]
    \centering
\includegraphics[width=0.43\linewidth]{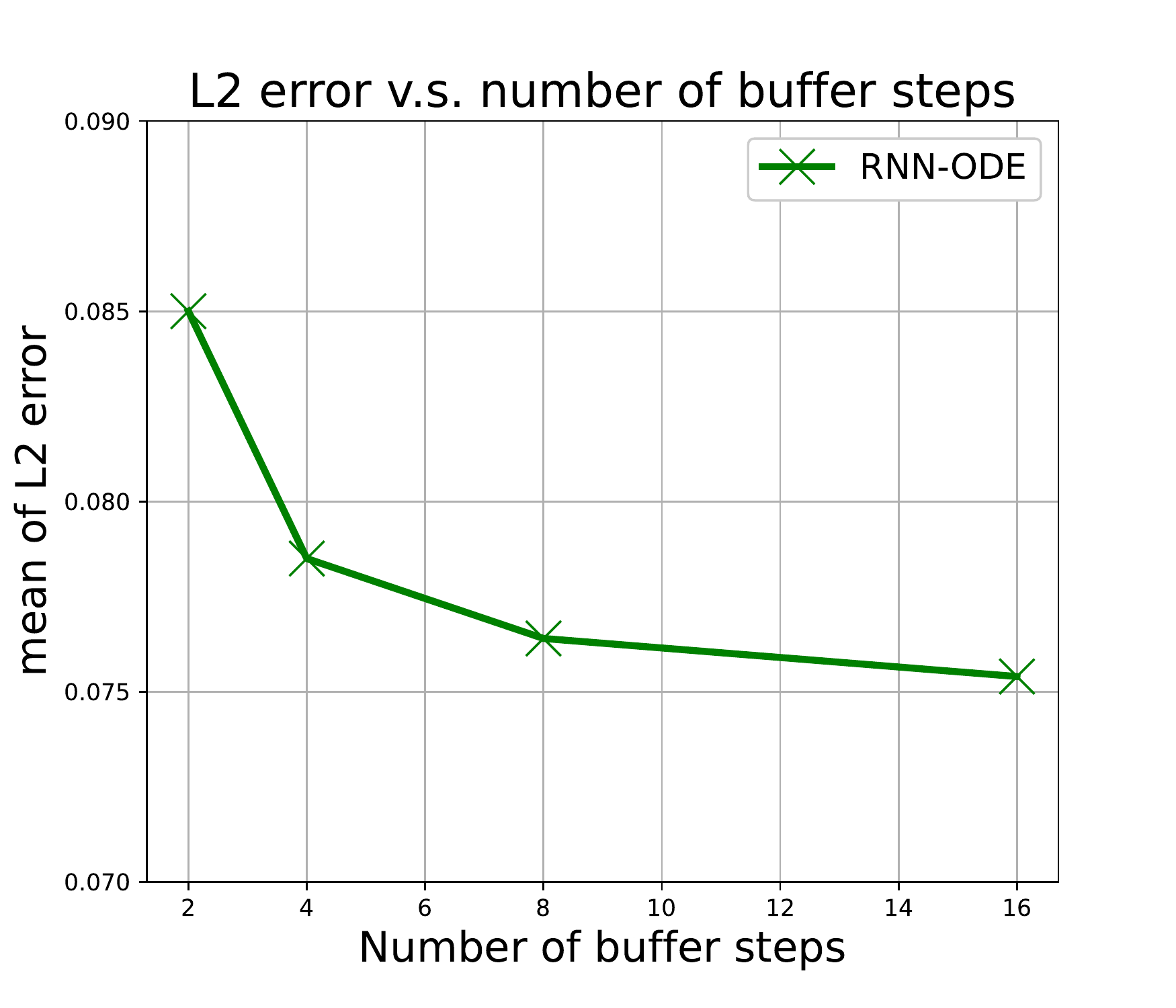}
    \caption{Comparison of the fitting errors versus the number of buffer steps for the discrete event-type data generated from the Hawkes process for RNN-ODE (the number of grids is 65).}
    \label{fig:Hawkes-buffer}
\end{figure}

\subsection{Other Implementation Details and Additional Results}\label{sec:add-exp-results}

We implement all the methods using PyTorch (Paszke et al., 2019), and all the experiments are run on a PC with 2.6 GHz 6-Core. We use Adam (Kingma \& Ba, 2014) for optimization. 
Moreover, additional numerical results (Boxplots for Section \ref{sec:numerical}) are given in Figures \ref{fig:spiral-error-boxplot}, \ref{fig-spiral-LSTM}, \ref{fig-spiral-Lipschitz}, \ref{fig-spiral-smaller-LSTM}, \ref{fig-Hawkes-const-deltat}, \ref{fig:ECG-error-boxplot}, and \ref{fig:ECG-example2}.

\paragraph{Boxplot of Figure \ref{fig:spiral-error}.} Figure \ref{fig:spiral-error-boxplot} shows the boxplot of the prediction errors of the models for the spiral data (the mean of MSE over replicas is plotted in Figure \ref{fig:spiral-error}). 

\begin{figure}[htpb]
    \centering
\includegraphics[width=0.5\linewidth]{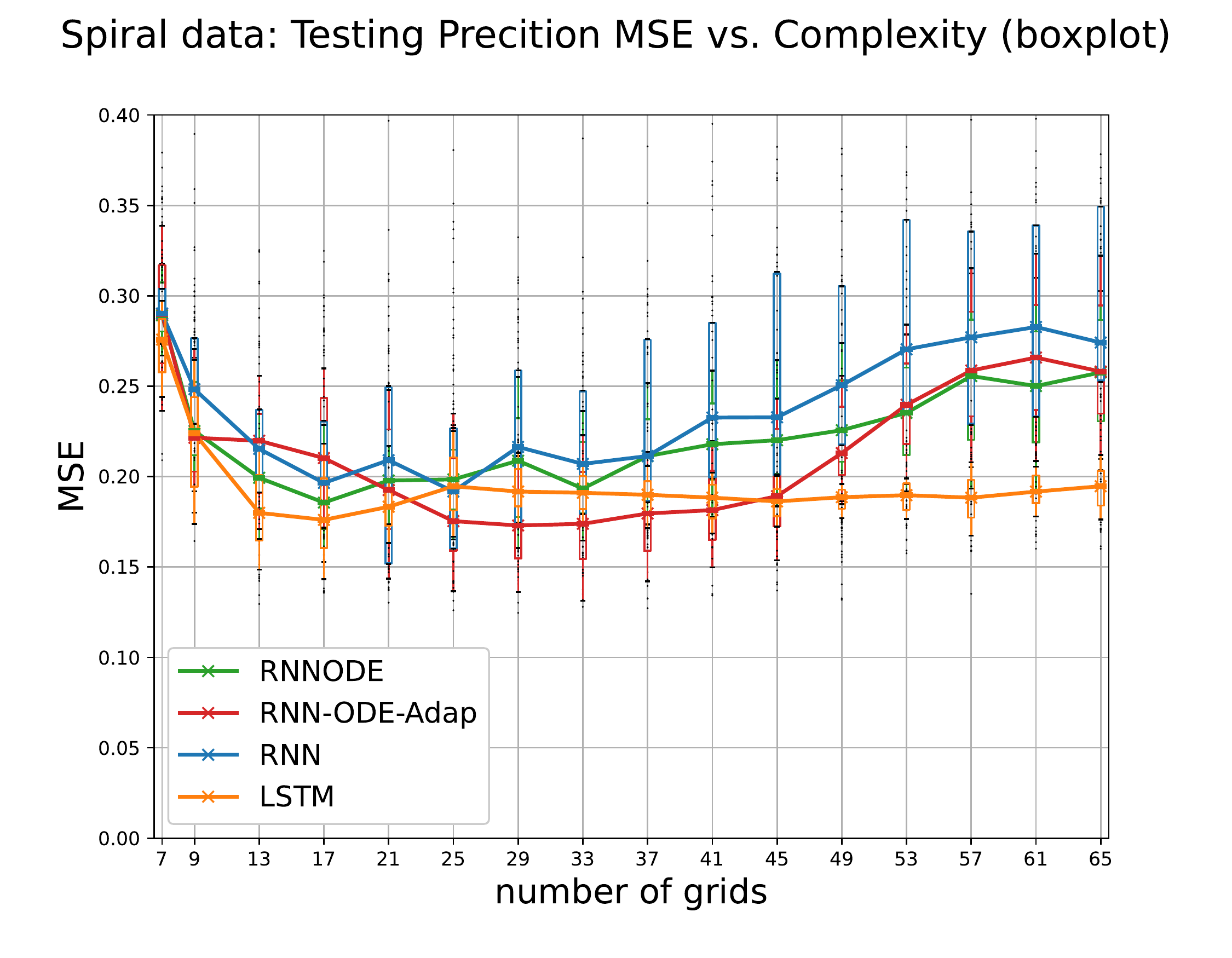}
    \caption{The boxplot of the prediction errors on the simulated spiral data from Eq. \ref{eq:spiral-ODE} for RNN, LSTM, RNN-ODE, and RNN-ODE-Adap. $x$ and $y$ axes have been explained in the caption of Figure \ref{fig:spiral-error}. }
    \label{fig:spiral-error-boxplot}
\end{figure}

\paragraph{LSTM and Lipschitz-RNN variants.} Figure \ref{fig-spiral-LSTM} shows the mean and the boxplot of the prediction errors of the models for the spiral data, including the LSTM variant of the adaptive model (which we refer to as \textit{LSTM-ODE-Adap} and  is plotted in the orange dashed lines). Similarly, Figure \ref{fig-spiral-Lipschitz} shows the mean and the boxplot of the prediction errors of the models for the spiral data, including the Lipschitz-RNN \cite{erichson2020lipschitz} and its adaptive variant (which we refer to as \textit{Lipschitz-RNN-Adap} and are plotted in the light and dark purple solid lines respectively). 

\begin{figure}[h]
\centering
\includegraphics[width=\textwidth]{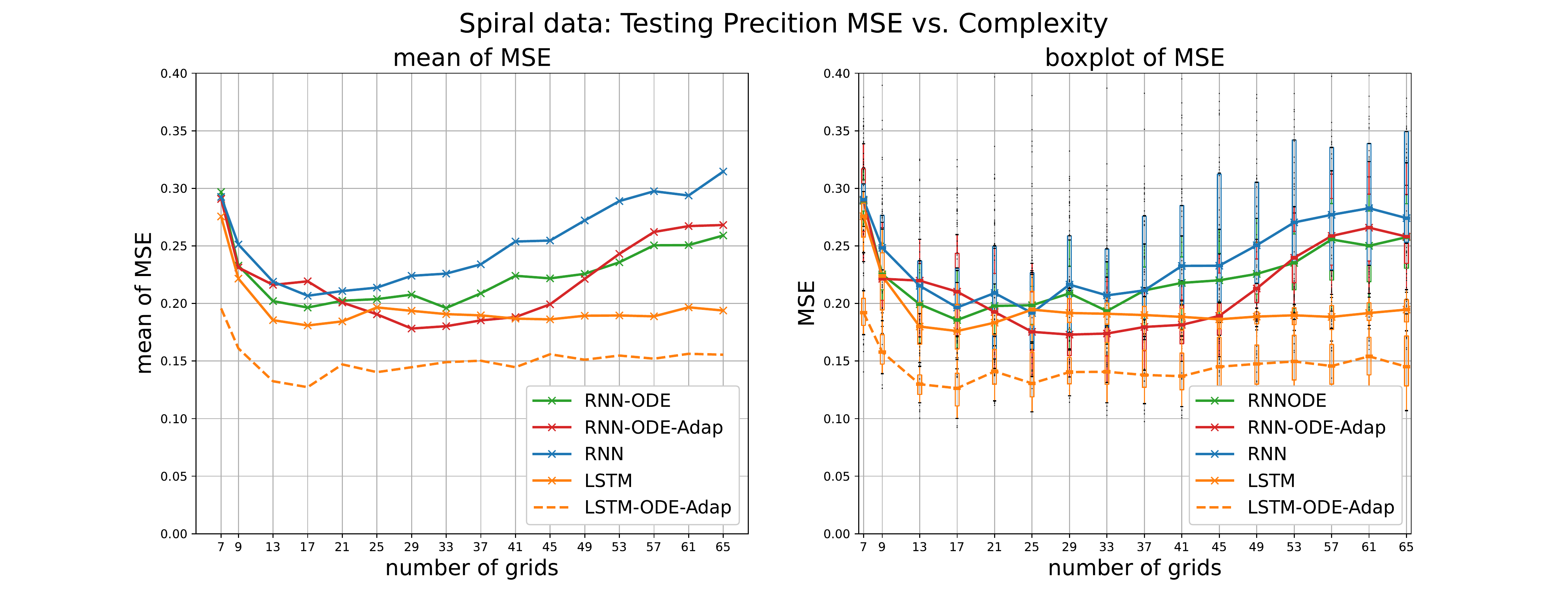}      
\caption{Comparison of prediction errors on the simulated spiral data from Eq. \ref{eq:spiral-ODE}, including the LSTM variant (plotted in the orange dashed lines). The left and right panels show the mean and boxplot of MSE, respectively. $x$ and $y$ axes have been explained in the caption of Figure \ref{fig:spiral-error}.}
\label{fig-spiral-LSTM}
\end{figure}

\begin{figure}[t]
\centering
\includegraphics[width=\textwidth]{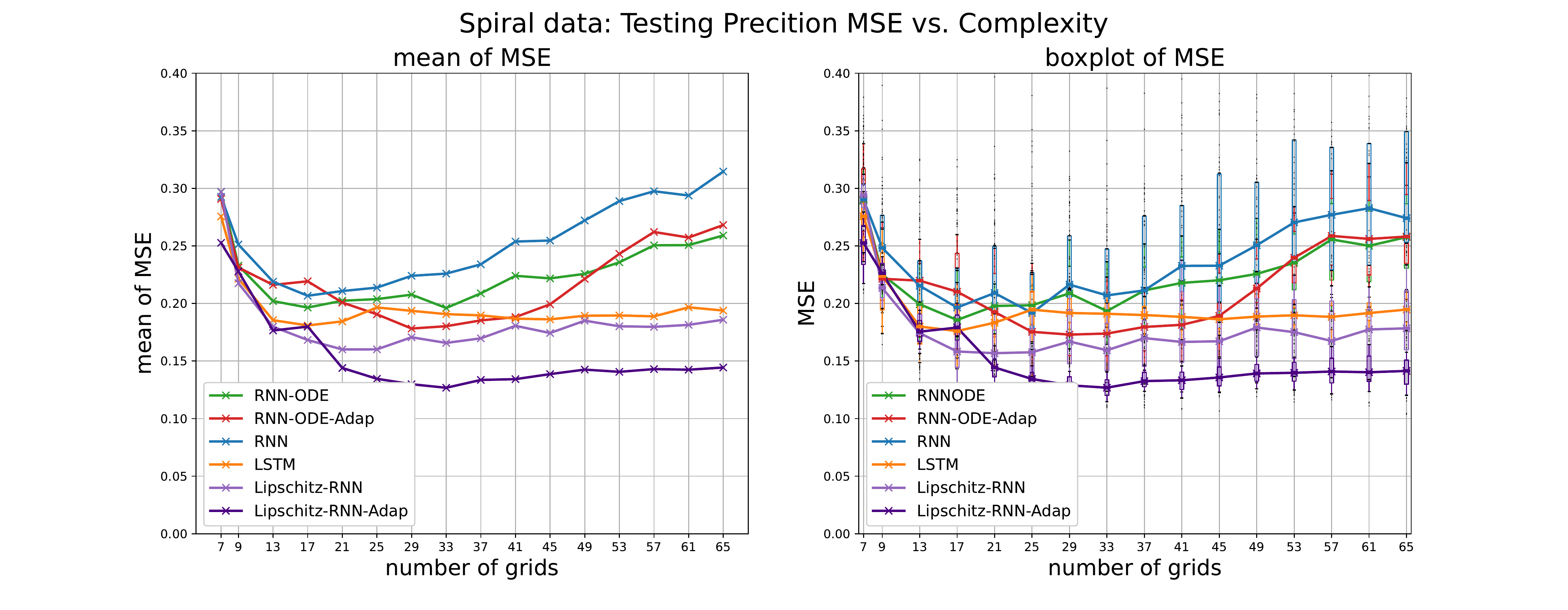}      
\caption{Comparison of prediction errors on the simulated spiral data from Eq. \ref{eq:spiral-ODE}, including the Lipschitz-RNN and its adaptive variant (plotted in the dark and light purple solid lines). The left and right panels show the mean and boxplot of MSE, respectively. $x$ and $y$ axes have been explained in the caption of Figure \ref{fig:spiral-error}.}
\label{fig-spiral-Lipschitz}
\end{figure}

The results in Figures \ref{fig-spiral-LSTM} and \ref{fig-spiral-Lipschitz} indicate that LSTM and Lipschitz-RNN with adaptive time steps achieve higher accuracy than the other models, thus validating the utility of incorporating adaptive time steps. Furthermore, this demonstrates that our proposed scheme of adaptive time steps can be easily and flexibly integrated into various time series models, leading to enhanced performance.

\paragraph{Sensitivity of LSTM to the number of parameters.} Figure \ref{fig-spiral-smaller-LSTM} shows the mean and the boxplot of the prediction errors of the models for the spiral data, including the LSTM with a similar number of parameters to that of RNN models. It can be observed that the performance of LSTMs with varying numbers of parameters is comparable and thus, the performance of LSTM is not sensitive to the number of parameters.

\begin{figure}[h]
\centering
\includegraphics[width=\textwidth]{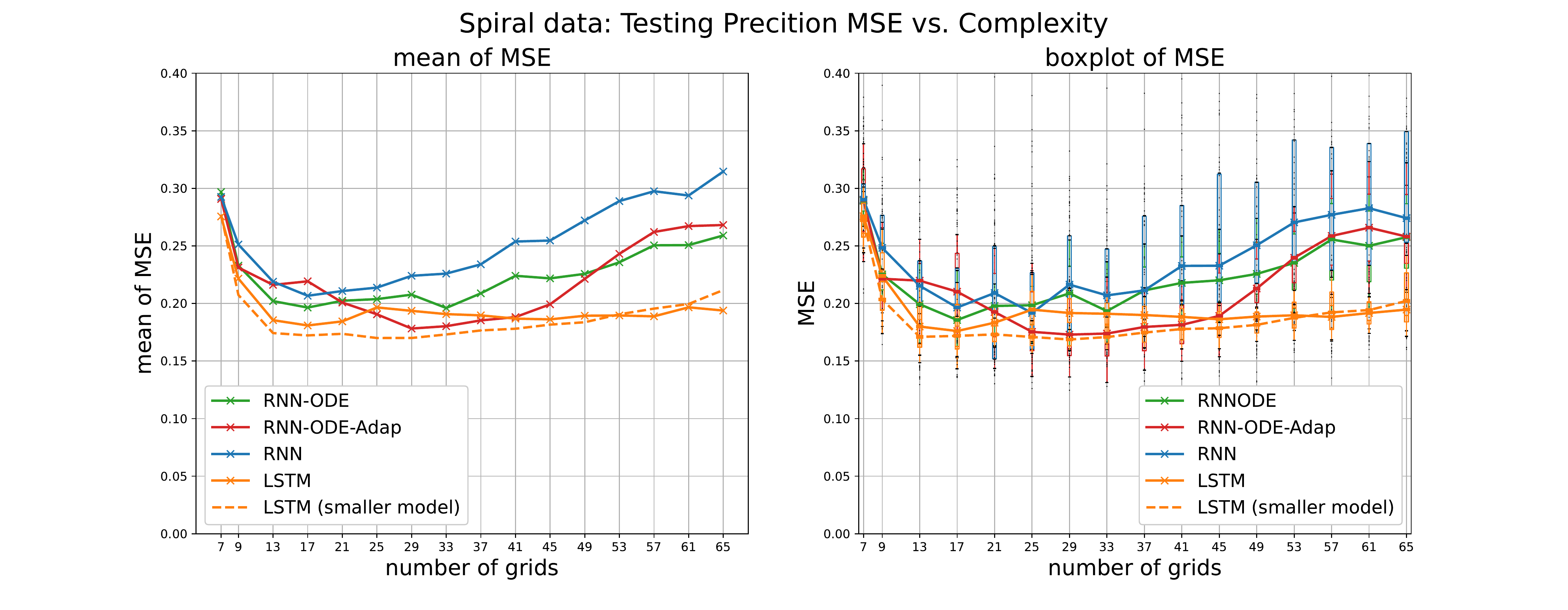}
\caption{Comparison of prediction errors on the simulated spiral data from Eq. \ref{eq:spiral-ODE}, including the LSTM with a similar number of parameters to that of RNN models (plotted in the orange dashed lines). The left and right panels show the mean and boxplot of MSE, respectively.}
\label{fig-spiral-smaller-LSTM}
\end{figure}

\paragraph{Ablation study of the time difference term in the training objective \eqref{eq:training_loss1}.} Figure \ref{fig-Hawkes-const-deltat} shows the boxplot of MSE of the models for the ablation study without the term $|t^{(\Tr,k)}_i-t^{(\Tr,k)}_{i-1}|$ on the event-type data, and RNN-ODE-Adap is plotted in a red dashed line. The results indicate that if the neural networks are trained without considering the time intervals, the models fail to fit the underlying intensity function, despite having the same network structure as before.

\begin{figure}[h]
\centering
\includegraphics[width=0.5\textwidth]{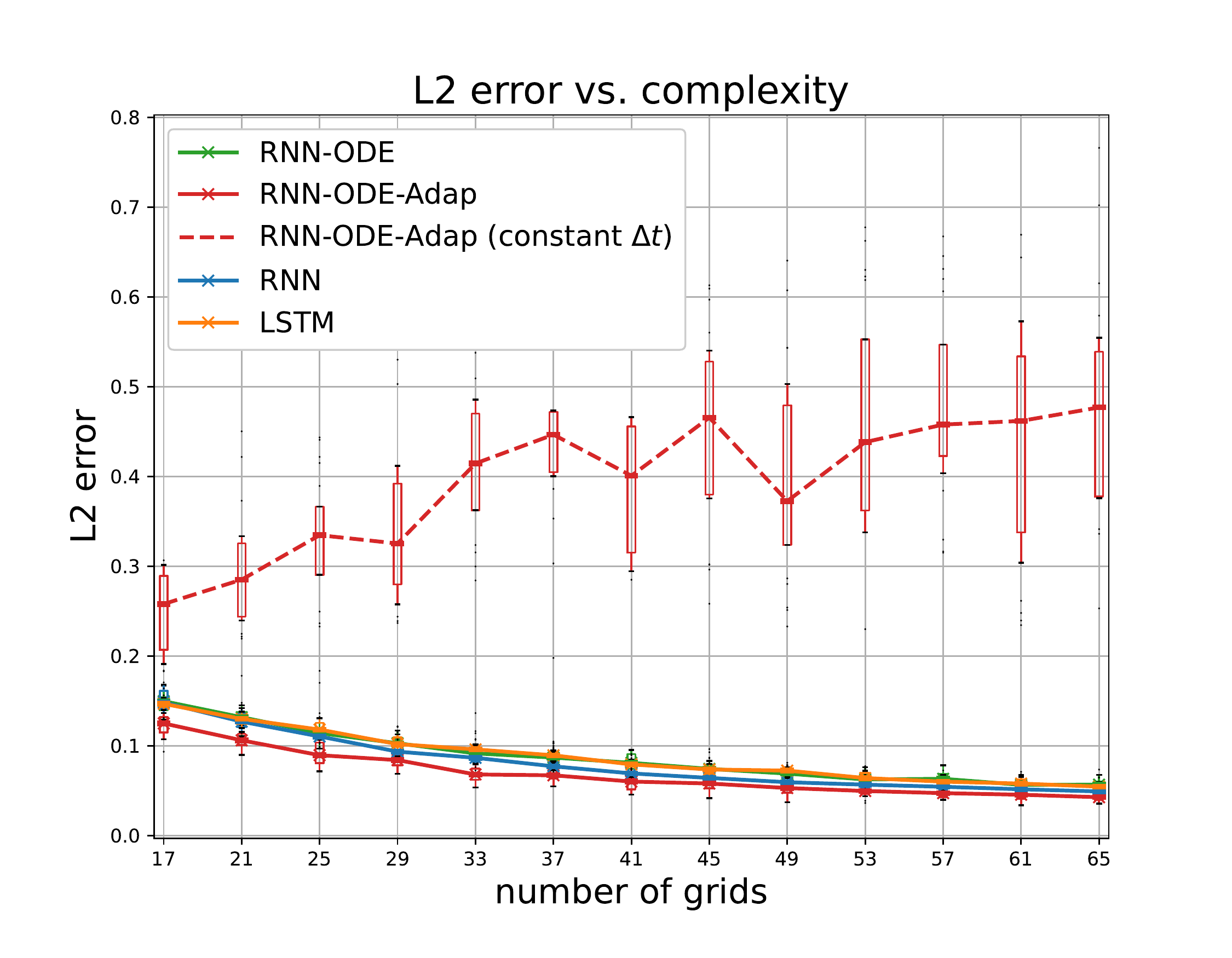}
\caption{Comparison of prediction errors for the event-type data generated from Hawkes processes, including RNN-ODE-Adap trained with constant $\Delta t$ (plotted in the red dashed line).}
\label{fig-Hawkes-const-deltat}
\end{figure}

\paragraph{Boxplot of Figure \ref{fig:ECG-error}.} Figure \ref{fig:ECG-error-boxplot} shows the boxplots of the prediction errors for ECG data, and the corresponding mean of MSE is shown in Figure \ref{fig:ECG-error}.
\begin{figure}[h]
    \centering
\includegraphics[width=0.9\linewidth]{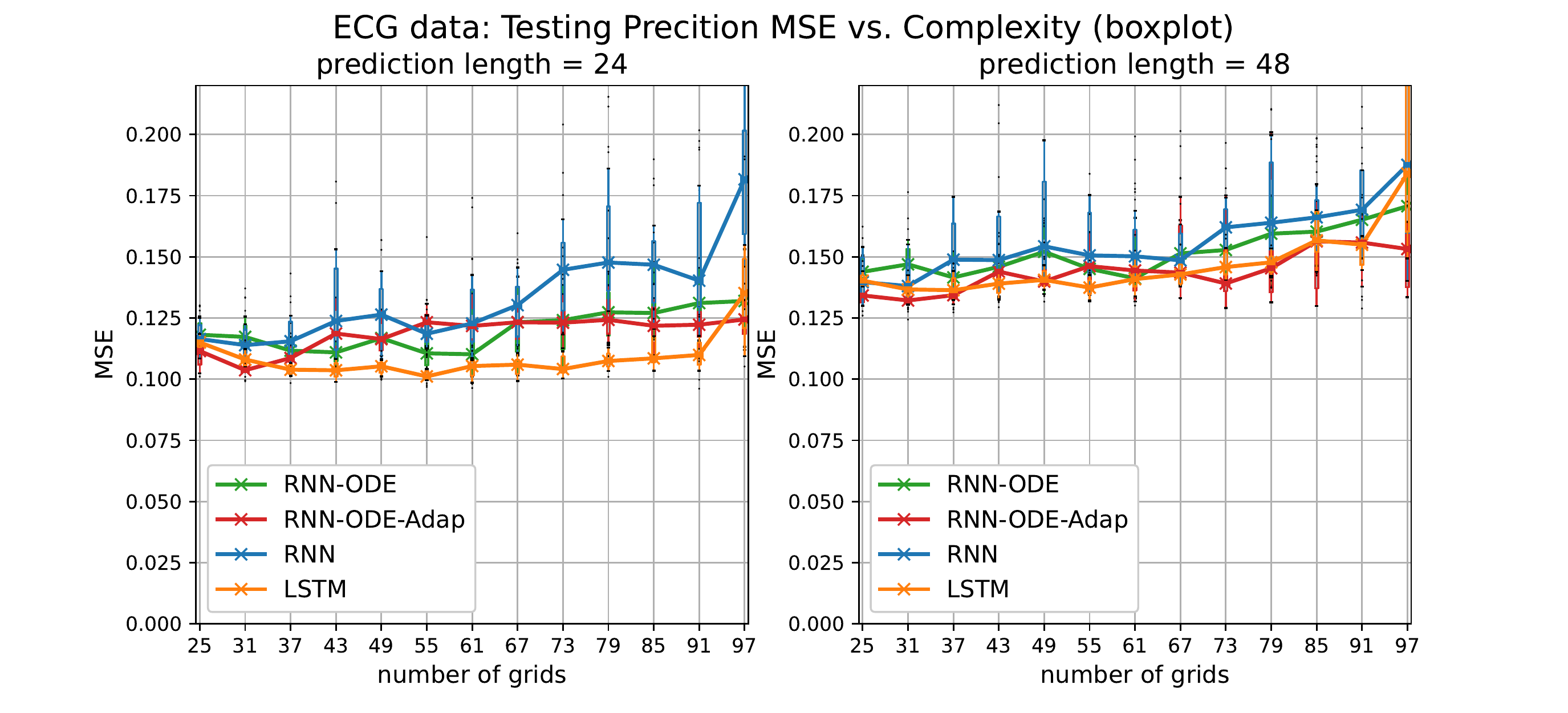}
    \caption{Boxplot of the prediction errors on the real ECG data.}
    \label{fig:ECG-error-boxplot}
\end{figure}

\paragraph{Examples of 24-step predictions for the ECG data. } Figure \ref{fig:ECG-example2} presents a comparison of 24-step ahead predictions for the testing ECG data using RNN and \modelname. The corresponding 48-step ahead predictions can be found in Figure \ref{fig:ECG-error} (b).

\begin{figure}[h]
\centering
  \includegraphics[width=0.8\linewidth]{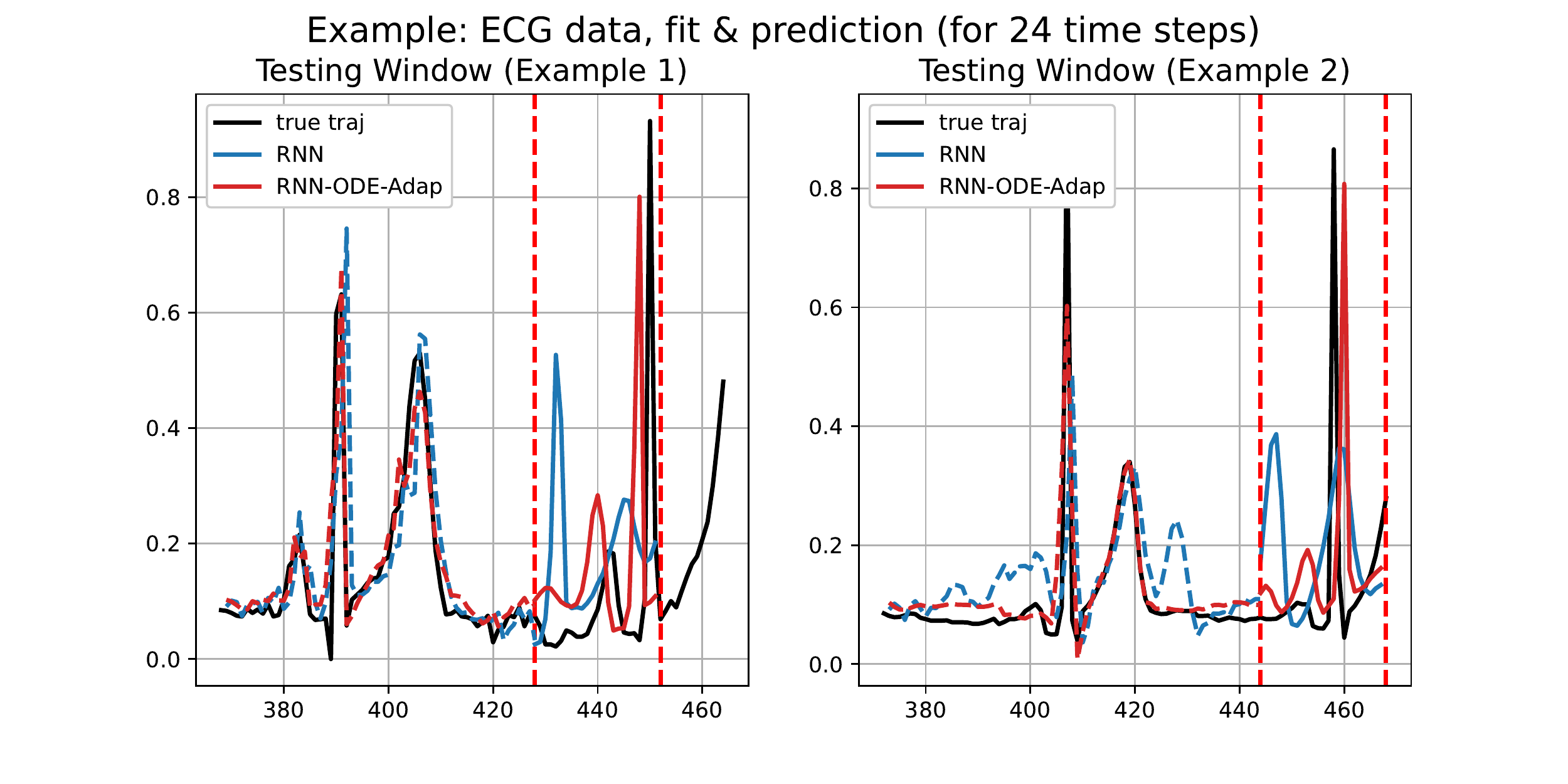}
    \caption{Examples of 24 steps ahead prediction for the testing ECG data using RNN (marked in blue) and {\modelname} (marked in red). The predicted region is marked between dashed lines.}
    \label{fig:ECG-example2}
\end{figure}

\paragraph{Comparison with LEM.} We compare our model with LEM \cite{rusch2021long_app}, which incorporates the time-adaptivity through the time modulator multiplied by the ODE function. The performance is evaluated on the simulated data generated from the FitzHugh-Nagumo system \cite{fitzhugh1955},
\[v' = v-\frac{v^3}{3}-w+I_{\text{ext}},\quad w'=\tau(v+a-bw),\]
which is a two-scale dynamical system and included as an example in \cite{rusch2021long_app}. As in \cite{rusch2021long_app}, we take $\tau=0.02, I_{\text{ext}}=0.5, a=0.7, b=0.8$, the time $t\in [0,200]$, and the initial values $(v_0,w_0)=(c_0,0)$, where $c_0\sim \mathcal{U}([-1,1])$. We rescale the system such that the time horizon is $[0,1]$ and $|v'|$ is $O(1)$. Specifically, if we formulate the original system as $y(t)' = f(y(t))$, where $y=(v,w)$, then we consider the rescaled system $\tilde{y}(\tau) = \tilde{f}(\tilde{y}(\tau))$, with $t=\beta\tau, \tilde{y}(\tau) = \alpha y(\beta\tau), \tilde{f}(\xi) = \alpha f(\frac{1}{\alpha}\xi)$. We take $\alpha=10,\beta=200$, and in this way $\tau\in [0,1]$. 

We compare the performance of the following models,
\begin{align*}
\text{LEM} &: \quad
\begin{cases}
h'(t) = \hat{\sigma}(W_2h(t)+V_2x(t)+b_2)\circ (\sigma(W_hg(t)+V_hx(t)+b_h)-h(t)),\\
g'(t) = \hat{\sigma}(W_1h(t)+V_1x(t)+b_1)\circ (\sigma(W_gh(t)+V_gx(t)+b_g)-g(t)),    
\end{cases}\\
\text{LEM-0} &: \quad
\begin{cases}
h'(t) = \sigma(W_hg(t)+V_hx(t)+b_h)-h(t),\\
g'(t) = \sigma(W_gh(t)+V_gx(t)+b_g)-g(t),    
\end{cases}\\
\text{RNN-ODE} &: \quad
\begin{cases}
h'(t) = \sigma(W_{hh}g(t) + W_{hg}g(t)+V_hx(t)+b_h),\\
g'(t) = \sigma(W_{gh}h(t) + W_{gg}g(t)+V_gx(t)+b_g),  
\end{cases}
\end{align*}
with the output
\[\hat{x}(t) = W_{x,h}h(t) + W_{x,g}g(t) + b_x.\]
Here $\hat{\sigma}(x)=0.5(1+\tanh(x/2)), h,g\in\R^{d_h}$.
Note that LEM-0 is LEM without the time modulators, and RNN-ODE possesses the vanilla RNN structure if we view the concatenated $(h(t), g(t))\in\R^{2d_h}$ as the hidden state. As in the other experiments, we still take $d_h=128$.

The length of both training and testing windows is set to $N=64$, with the windows sampled at regular time intervals defined by $t_i = \frac{i}{N-1}, i=0,\dots,N$. When integrating the ODE models, the time step difference is kept consistent with the physical time difference, specified as $\Delta t = \frac{1}{N-1}$. Consequently, the only difference between the models is the structure of the neural ODE. To compare the time adaptivity incorporated in the time modulator of LEM and the adaptive algorithm proposed in this study, we additionally present the results obtained by training LEM-0 and RNN-ODE with the adaptively chosen time steps.
\begin{table}[b]
    \centering
    % \label{tab-notations&parameters}
    \begin{threeparttable}   
    \begin{tabular}{ccccc}
    \toprule
    Model & Training data & Testing MSE \\
    \midrule
    LEM-0 & Original windows, $N=64$ & $3.54e-02 \pm 2.69e-03$ \\
    LEM-0 & Adaptive windows, $\bar{N}_a = 43$ & $3.61e-02 \pm 1.85e-03$ \\
    RNN-ODE & Original windows, $N=64$ & $2.55e-02 \pm 4.12e-03$ \\
    RNN-ODE & Adaptive windows, $\bar{N}_a = 43$ & $2.44e-02 \pm 5.56e-03$ \\
    LEM \cite{rusch2021long_app} & Original windows, $N=64$ & $3.03e-02 \pm 2.29e-03$ \\
    \bottomrule
    \end{tabular}
    \end{threeparttable}
    % \vspace{0.5em}
    \caption{MSE for one-step prediction of the models on data simulated from the FitzHugh-Nagumo system \cite{fitzhugh1955}. The presented results are from 25 replicas, and the MSE for the one-step prediction is calculated as in \eqref{eq:fit-error}. }
    \label{tab:FN-error}
\end{table}

Table \ref{tab:FN-error} presents the MSE for one-step predictions made by the models, trained either using the original training windows or those selected adaptively (Both LEM and LEM-0 are implemented utilizing the code in \cite{rusch2021long_app}). As observed from Table \ref{tab:FN-error}, the RNN-ODE models exhibit a lower error on average compared to the LEM models. The adaptive training windows slightly enhance the performance of the RNN-ODE model, while they do not improve the performance of LEM-0. LEM outperforms LEM-0 by incorporating time modulators, yet the RNN-ODE-Adap model behaves better than LEM by up to one standard deviation.

\end{document}